\documentclass{article}
\usepackage[top=1in, bottom=1in, left=1in, right=1in]{geometry} 
\usepackage{setspace} 
\usepackage{amsfonts}
\usepackage{amsmath}
\usepackage{amsthm} 
\usepackage{amssymb} 
\usepackage{epsfig}
\usepackage{url}
\usepackage{bm}
\usepackage{bbm}
\usepackage{booktabs}
\usepackage{xcolor}
\usepackage{adjustbox}
\usepackage{mathtools} 
\usepackage{siunitx}
\usepackage{tikz-cd}
\tikzcdset{ampersand replacement=\&}
\usepackage{graphicx}
\usepackage{subcaption}
\usepackage[colorlinks,pagebackref=true]{hyperref}
\usepackage{cite}
\usepackage{float} 
\usepackage{subfiles}
\usepackage{longtable}
\usepackage{algorithm}
\usepackage{algpseudocode}
\usepackage{todonotes}
\setuptodonotes{inline} 
\usepackage{subfiles}

\usepackage{tikz}
\usetikzlibrary{decorations.pathmorphing, positioning}

\DeclareMathOperator{\diag}{diag}
\DeclareMathOperator{\vect}{vec}

\DeclareMathOperator{\rank}{rank}

\newcommand{\e}{{\rm e}}
\newcommand{\E}{{\mathbb E}}

\newcommand{\Q}{{\mathbb Q}}

\newcommand{\R}{{\mathbb R}}

\newcommand{\N}{{\mathbb N}}

\newcommand{\Dcal}{{\mathcal D}}

\newcommand{\Fcal}{{\mathcal F}}

\newcommand{\Hcal}{{\mathcal H}}
\newcommand{\Ical}{{\mathcal I}}

\newcommand{\Ucal}{{\mathcal U}}
\newcommand{\Vcal}{{\mathcal V}}

\newtheorem{proposition}{Proposition}[section]
\newtheorem{lemma}[proposition]{Lemma}
\newtheorem{theorem}[proposition]{Theorem}
\newtheorem{definition}[proposition]{Definition}

\newtheorem{remark}[proposition]{Remark}

\newtheorem{exampleemph}[proposition]{Example}   
\newenvironment{example}{\begin{exampleemph}\begin{upshape}}{\end{upshape}\end{exampleemph}} 

\newcommand\tcapfig[1]{\captionsetup{position=top, font=normalsize, labelfont=bf, textfont=normalfont, justification=centering, margin=0mm, aboveskip=2mm, belowskip=0mm, labelsep=colon, singlelinecheck=false}\caption{#1}}
\newcommand\bnotefig[1]{\captionsetup{position=bottom, font=footnotesize,  textfont=normalfont, margin=1mm, skip=2mm, justification=justified, singlelinecheck=false}\caption*{#1}}
\begin{document}

\title{Transfer Learning Across Fixed-Income Product Classes}
\author{ Nicolas Camenzind\footnote{EPFL, Switzerland. Email: nicolas.camenzind@epfl.ch} \and Damir Filipovi\'c\footnote{EPFL, Switzerland. Email: damir.filipovic@epfl.ch}} 
\date{13 January 2026}
\maketitle

\begin{abstract}
We propose a framework for transfer learning of discount curves across different fixed-income product classes. Motivated by challenges in estimating discount curves from sparse or noisy data, we extend kernel ridge regression (KR) to a vector-valued setting, formulating a convex optimization problem in a vector-valued reproducing kernel Hilbert space (RKHS). Each component of the solution corresponds to the discount curve implied by a specific product class. We introduce an additional regularization term motivated by economic principles, promoting smoothness of spread curves between product classes, and show that it leads to a valid separable kernel structure. A main theoretical contribution is a decomposition of the vector-valued RKHS norm induced by separable kernels. We further provide a Gaussian process interpretation of vector-valued KR, enabling quantification of estimation uncertainty. Illustrative examples show how transfer learning tightens confidence intervals compared to single-curve estimation. An extensive masking experiment demonstrates that transfer learning significantly improves extrapolation performance. 
\end{abstract}

\vspace{2ex}


\vspace{2ex}



\section{Introduction}
\label{sec:introduction}

We introduce a framework for transfer learning of discount curves across different fixed-income product classes. Since discount curves are inherently unobservable, they must be inferred from the observable prices of fixed-income instruments. A key feature of the proposed framework is its ability to incorporate complementary market information across product classes. Accurate estimation is critical, as discount curves are fundamental to finance, providing the basis for appropriately discounting future cash flows. Consequently, their precise estimation holds significant practical relevance.

Numerous methods have been proposed for single discount curve estimation. Classical approaches include the parametric Nelson--Siegel--Svensson model~\cite{Nelson1987, SVENSSON1944, GURKAYNAK20072291}, as well as nonparametric methods such as Fama--Bliss~\cite{FAMABLISS1987}, Smith--Wilson~\cite{smi_wil_01}, and Liu--Wu~\cite{LIU20211395}. More recently, \cite{Filipovic2022} introduced a kernel ridge regression (KR) framework, providing a theoretically grounded solution based on reproducing kernel Hilbert space (RKHS) theory. KR yields a closed-form, linear estimator and empirically outperforms benchmark models for US and Swiss government bonds~\cite{Filipovic2022, CamenzindFilipovic2024}. However, like other methods, KR struggles with extrapolation in maturity ranges where data are sparse or absent~\cite{CamenzindFilipovic2024}.

This limitation motivates the use of transfer learning~\cite{Weiss2016-kn, Pan2010-qg}, a well-established concept in machine learning that is closely related to multitask learning~\cite{Caruana1997}. Transfer learning seeks to improve estimation by jointly solving related problems and sharing information across them, particularly when data for the primary task are limited, noisy, or costly to obtain. In the context of discount curves, transfer learning arises naturally across fixed-income product classes, with suitable adjustment for cross-currency effects. A product class refers to a group of fixed-income instruments priced using a common discount curve, denominated in the same currency and characterized by similar risk features such as issuer type, collateralization, or credit quality. Examples include government bonds issued by the same sovereign, interest-rate swaps referencing a common overnight risk-free rate~\cite{Schrimpf2019, SixSaron, SNBSaron, ECBEstr, BOESoniaOverhaul, FedSofr}, and corporate bonds within a given credit rating class.

This paper develops a theoretical framework for transfer learning of discount curves followed by an extensive empirical masking experiment that demonstrates the economical significant benefits of it. The experiment is centered around transfer learning between US government bonds and SOFR swaps.\footnote{SOFR stands for Secured Overnight Financing Rate and is the new overnight risk-free reference rate in the US, see \cite{FedSofr}}\footnote{Another study is in progress on transfer learning government bonds across currencies.} Our methodology generalizes to any set of fixed-income products that can be represented jointly under a discounted cash flow framework. Although limits to arbitrage may cause different product classes to imply distinct discount curves even when all are considered risk-free~\cite{wu_jar_24}, we show that the discounted cash flow principle can be naturally embedded into an arbitrage-free pricing framework.

We formulate the transfer learning problem as a vector-valued KR, leading to a convex optimization problem in a vector-valued RKHS. The objective function balances pricing errors against the smoothness of the resulting discount curve, as induced by the norm associated with an operator-valued kernel. This norm serves as a regularization term that ensures a well-posed problem and mitigates overfitting. Each component of the solution corresponds to the discount curve implied by a specific product class. Analogous to the scalar case, we derive a closed-form expression for the vector-valued KR estimator.

The theory of vector-valued RKHS is well-established~\cite{pau_rag_16, pontil:on_learning_vvf}, with operator-valued kernels, such as matrix-valued kernels in~$\mathbb{R}^n$, playing a central role~\cite{kadri_duflos_2016}. Fundamental results from RKHS theory, including the representer theorem and Moore's theorem, extend naturally to the vector-valued setting~\cite{pau_rag_16}. A particularly tractable subclass, separable kernels, has been extensively studied~\cite{baldassarre:multi_output_learning, sheldon:graph_multi_task_learning, pontil:kernels_multi_task_learning, alv_ros_law_11}. Separable kernels are constructed as the product of a scalar kernel and a constant covariance matrix, the latter encoding the transfer learning structure. They offer computational advantages, including simple computation of inner products and induced norms~\cite{baldassarre:multi_output_learning}.

Building on the scalar case, we introduce an additional regularization term motivated by economic principles, penalizing the spread between discount curves across product classes. A main theoretical contribution of our work is a decomposition of the norm induced by separable kernels, generalizing a result from~\cite{baldassarre:multi_output_learning}. We prove that the resulting regularization yields a valid separable kernel, specifically tailored to our transfer learning problem. Rather than enforcing identical curves, the regularization promotes smoothness of spread curves under an economically motivated norm. This connects naturally to graph regularization techniques~\cite{Smola2003, sheldon:graph_multi_task_learning}.

We further provide a Gaussian process~\cite{Chen2019} interpretation of the vector-valued KR, enabling quantification of estimation uncertainty for the discount curves. In doing so, we extend the well-known correspondence between KR and Gaussian processes in the scalar case~\cite{Rasmussen2005-rn} to the setting of transfer learning. Our illustrative examples show that the transfer learning framework significantly tightens confidence intervals around the estimated discount curves.

The remainder of the paper is organized as follows. Section~\ref{sec:transfer_learning_problem_formulation} formulates the transfer learning problem for discount curves and presents the representation theorem essential for implementation. Section~\ref{sec:multivariate_gp_vies} develops the Gaussian process perspective. Section~\ref{sec:sep_kernels} introduces separable kernels as a natural class of matrix-valued kernels for our setting. Section~\ref{sec:std_product_formulation} discusses how standard fixed-income products can be embedded into the transfer learning formulation. Section ~\ref{sec:empirical_illustrations} covers the applied part where we focus on US government bonds and SOFR swaps. We first introduce the data and perform hyperparameter selection, show illustrative yield and forward curves based on transfer learning before performing an in depth masking experiment which shows the economical significant effects of transfer learning. The Appendix provides a self-contained introduction to the theory of vector-valued RKHS, collects all proofs, and details the embedding of the discounted cash flow principle into an arbitrage-free pricing framework.

\section{Transfer Learning Problem Formulation}
\label{sec:transfer_learning_problem_formulation}

In this section, we present the general problem formulation for transfer learning of discount curves across $A$ different fixed-income product classes. Our framework requires only that the theoretical price of a fixed-income product be expressed as the sum of its discounted cash flows.

Specifically, for every product class $a=1,\dots,A$, there are $M_a$ fixed-income instruments with common cash flow dates $0<x_1<\cdots <x_N$, stacked into the column vector $\bm x = (x_1,\dots,x_N)^\top$.\footnote{Cash flow dates $\bm x$ are assumed to be common across all product classes without loss of generality.} The total number of instruments is given by $M=M_1+\cdots+M_A$. For each instrument we observe noisy ex-coupon prices, 
$P_a = (P_{a,1}, \dots, P_{a,M_a})^\top$.
We denote the associated $M_a\times N$ cash flow matrix by $C_a=(C_{a,ij})$ where $C_{a,ij}$ is the cash flow of instrument $i$ of product class $a$
that occurs in $x_j$.

In line with the discounted cash flow principle, we assume that for every product class $a$, there exists a unique discount curve $g_a:[0,\infty)\to \R$ with $g_a(0)=1$ and such that the price of every instrument $i$ in product class $a$ is given by
\begin{equation}
\label{eq:pricing_eq}
   P_{a,i} = \sum_{j=1}^N C_{a,ij} g_a(x_j). 
\end{equation}

The objective of this paper is to jointly estimate the discount curves $g = (g_1,\dots,g_A)^\top$ from observed market prices $P_a$. To this end, we decompose each curve $g_a$ as the sum of an exogenous prior function $p_a$ and a hypothesis function $h_a$, that is,
\[
g_a = p_a + h_a \quad \text{for all $a = 1,\dots,A$.}
\]
Here, the prior $p = (p_1,\dots,p_A)^\top : [0,\infty) \to \mathbb{R}^A$ is assumed to satisfy $p(0) = 1$, and the hypothesis $h = (h_1,\dots,h_A)^\top : [0,\infty) \to \mathbb{R}^A$ is constrained to satisfy $h(0) = 0$.\footnote{This additive specification mirrors the structure of linear-rational term structure models; see \cite{Filipovic2017-wh}.} A natural and simple choice for the prior is the constant function $p = 1$. 

We model $h$ as an element of a vector-valued RKHS $\mathcal{H}$ over the domain $E = [0,\infty)$, taking values in $\mathbb{R}^A$ and satisfying $h(0) = 0$ for all $h \in \mathcal{H}$. The associated reproducing kernel is a matrix-valued function $K : [0,\infty) \times [0,\infty) \to \mathbb{R}^{A \times A}$. Appendix~\ref{secRKHS} provides a self-contained introduction to the theory of vector-valued RKHS, including its foundational properties and practical relevance for our setting.

To enable matrix notation, we introduce the following conventions.  For any function $f$, we write $f(\bm x)=(f(x_1),\dots,f(x_N))^\top$ for the corresponding array of function values. For a general matrix $Q\in\mathbb{R}^{m\times n}$, we write $Q_i=(Q_{i1},\dots,Q_{in})$ for its $i$-th row vector, and define the vectorization of $Q$ as the vector obtained by stacking its columns, $\vect(Q)= (Q_{11},\dots,Q_{m1},Q_{12},\dots,Q_{m2},\dots,Q_{1n},\dots,Q_{mn})^\top \in\R^{nm}$. Accordingly, we denote the matrix $h^\top(\bm x)= (h_1(\bm x),\dots, h_A(\bm x))\in\R^{N\times A}$, and we obtain the vector
\[ \vect(h^\top(\bm x)) =
  \begin{pmatrix}
    h_1(\bm x)\\ \vdots \\ h_A(\bm x)
  \end{pmatrix} =(h_1(x_1),\dots,h_1(x_N),h_2(x_1),\dots,h_2(x_N),\dots,h_A(x_1),\dots,h_A(x_N))^\top\in \R^{AN}.\]

We also stack the cash flow matrices and price vectors across product classes as
\[
\bm{C} =
\begin{pmatrix}
C_1 &        &        \\
    & \ddots &        \\
    &        & C_A
\end{pmatrix} \in \mathbb{R}^{M \times AN}, \quad 
\bm{P}=
  \begin{pmatrix}
    P_1  \\\vdots\\
    P_A
  \end{pmatrix} \in\mathbb{R}^{M},
\]
where \( \bm{C} \) is block diagonal with the individual cash flow matrices \( C_a \) along the diagonal, and all off-diagonal blocks equal to zero. The discounted cash flow equation \eqref{eq:pricing_eq} then reads $\bm P =\bm C \vect( p^\top(\bm x) + h^\top(\bm x))$. Including pricing errors $\bm\epsilon$ leads to
\begin{equation}
\label{eq:pricing_eq_error}
  \bm P= \bm C \vect( p^\top(\bm x) + h^\top(\bm x))+ \bm\epsilon. 
\end{equation}
Such pricing errors occur due to market imperfections and data errors.

The estimation objective reduces to finding a function \( h \in \mathcal{H} \) that balances the tradeoff between the weighted mean-squared pricing error,
\[
\sum_{a=1}^A \sum_{i=1}^{M_a} \omega_{a,i} \big( P_{a,i} - C_{a,i} p_a(\bm{x}) - C_{a,i} h_a(\bm{x}) \big)^2,
\]
and the regularity of \( h \) as quantified by adding the term $\lambda \|h\|_\Hcal^2$, for the vector-valued RKHS norm \( \|h\|_{\mathcal{H}} \) and a regularity parameter $\lambda>0$. This leads to the vector-valued KR problem 
\begin{equation}
    \label{eq:krr_obj}
    \min_{h\in\mathcal{H}} \Bigl\{\sum_{a=1}^A \sum_{i=1}^{M_a} \omega_{a,i} ( P_{a,i}-C_{a,i}p_a(\bm x) -C_{a,i} h_a(\bm x))^2 +\lambda\|h\|^2_{\mathcal{H}} \Bigr\}.
  \end{equation}
The weights \( \omega_{a,i} > 0 \) are exogenously specified and reflect the relative importance of the pricing terms. By the vector-valued kernel representer theorem, there exists a unique solution, which admits a closed-form expression in terms of the kernel matrix
\begin{equation}\label{eq:kernel_matrix}
  \bm{K}=
  \begin{pmatrix}
    \bm{K_{11}} & \dots  & \bm{K_{1A}}\\
    \vdots & \ddots & \vdots \\
    \bm{K_{A1}} & \cdots & \bm{K_{AA}}
  \end{pmatrix} \in\mathbb{R}^{AN\times AN},
\end{equation}
where each block \( \bm{K_{ab}} \in \mathbb{R}^{N \times N} \) has entries \( \bm{K}_{\bm{ab},ij} = K_{ab}(x_i, x_j) \). The following theorem formalizes this.

\begin{theorem}
  \label{thm:vect_krr}
  There exists a unique solution of the vector-valued KR problem \eqref{eq:krr_obj}, which is given by $\bar h = \sum_{j=1}^N K(\cdot,x_j) \beta_j$ where $\beta=(\beta_1,\dots,\beta_N)\in \R^{A\times N}$ takes the form
  \[ \vect( \beta^\top) = \bm C^\top\left( \bm C\bm K\bm C^\top + \bm \Lambda\right)^{-1}(\bm P - \bm C \vect( p^\top(\bm x))),\]
  for the block diagonal matrix $\bm\Lambda=\operatorname{diag}(\Lambda_1,\dots,\Lambda_A)\in\mathbb{R}^{M\times M}$ with $\Lambda_a=\operatorname{diag}(\lambda/\omega_{a,1},\dots,\lambda/\omega_{a,M_a})$. The corresponding discount curves are given by $\bar g =  p +\bar h$.
\end{theorem}

A common choice for the weights $\omega_{a,i}$, see, e.g.,~\cite{Filipovic2022, CamenzindFilipovic2024}, is based on the duration of the underlying instruments as presented in the following example.

\begin{example}
\label{ex:weights}
For any fixed-income instrument~$i$ of product class $a$ with cash flows~$C_{a,ij}$ at dates~$x_j$, its price as a function of yield-to-maturity (YTM)~$Y$ is given by
\[
Y \mapsto \Pi_{a,i}(Y) = \sum_{j=1}^N C_{a,ij} e^{-Y x_j}.
\]
The market-implied YTM~$Y_{a,i}$ is defined by $\Pi_{a,i}(Y_{a,i}) = P_{a,i}$, where $P_{a,i}$ denotes the observed market price. The model-implied YTM~$Y_{a,i}^g$, based on the discount curves~$g$, satisfies $\Pi_{a,i}(Y_{a,i}^g) = P_{a,i}^g = C_{a,i} g_a(\bm{x})$, consistent with the discounted cash flow equation~\eqref{eq:pricing_eq}. Using a first-order approximation, $P_{a,i}^g - P_{a,i} \approx \Pi_{a,i}'(Y_{a,i})(Y_{a,i}^g - Y_{a,i})$, we express the squared YTM error as an approximately weighted squared price error,
\[
(Y_{a,i}^g - Y_{a,i})^2 \approx \frac{1}{\left(\Pi_{a,i}'(Y_{a,i})\right)^2}(P_{a,i}^g - P_{a,i})^2.
\]
YTM is often used to compare fixed-income instruments across maturities. Weighting squared price errors by~$\omega_{a,i} = \frac{1}{M} \frac{1}{\left(\Pi_{a,i}'(Y_{a,i})\right)^2}$ in \eqref{eq:krr_obj} therefore ensures that estimation errors are more uniformly comparable across the maturity spectrum. See also Figure~\ref{fig:bid_ask} below.
\end{example}
 
\begin{remark}
 Theorem \ref{thm:vect_krr} can be extended towards infinite weights $\omega_{a,i}=\infty$, with the convention
  $\lambda/\infty = 0$, which corresponds to an exact fit of $P_{a,i}$, for selected $a$, $i$. This requires that the corresponding block of $\bm C\bm K\bm C^\top$ is invertible. See \cite[Theorem A.1]{Filipovic2022} for details.   
\end{remark}

\begin{remark}
Theorem~\ref{thm:vect_krr} remains valid even when no quotes are available for a given product class~$a$, i.e., when $M_a = 0$. In this case, the corresponding rows in $\bm{C}$ and $\bm{P}$ are omitted, and we adopt the convention that $\sum_{i=1}^0 = 0$. Remarkably, the solution curve $\bar h_a$ still depends on the other product classes via the joint regularization term in \eqref{eq:krr_obj}. In the extreme case where no quotes are available at all, $M = 0$, the solution $\bar{h}$ is identically zero, and the resulting discount curves reduce to the priors, $\bar{g} = p $.
\end{remark}

\section{Gaussian Process View}
\label{sec:multivariate_gp_vies}
Similar to the scalar case one can develop a Gaussian process perspective of the kernel ridge regression in the vector-valued case. We first discuss the general case and then specialize to separable kernels.

\subsection{Vector-valued Gaussian processes}

We recap the theory of vector-valued Gaussian processes and prove the equivalence of the posterior mean function and the vector-valued KR solution. We denote by $\mathcal{N}(m,\Sigma)$ the multivariate normal distribution with mean vector $m$ and covariance matrix $\Sigma$.

\begin{definition}[vector-Valued Gaussian Process]
\label{def:vgp}
  We say $g:E\to\mathbb{R}^A$ is a vector-valued Gaussian process with mean function $m=(m_1,\dots,m_A)^\top:E\to\mathbb{R}^A$ and kernel function
  $K(x,y):E\times E\to\mathbb{R}^{A\times A}$ if and only if for any $\bm x = (x_1,\dots,x_{N})^\top$
  \begin{equation*}
    \vect(g^\top(\bm x))\sim\mathcal{N}(\vect(m^\top (\bm x)),\bm K)
  \end{equation*}
  with $m^\top (\bm x)=(m_1(\bm x),\dots, m_A(\bm x))\in\mathbb{R}^{N\times A}$ and $\bm K$ as in \eqref{eq:kernel_matrix}.
  In this case we write $g\sim\mathcal{MG}(m, K)$.
\end{definition}
\begin{remark}
  There is no restriction to use $\bm x$ across all components of $g$. One can formulate a Gaussian process for any finite collection
  of points $\{\bm x_1,\dots,\bm x_n\}$, $\bm x_i\in\mathbb{R}^N$, such that $(g_1(\bm x_1),\dots,g_A(\bm x_n))\in\mathbb{R}^{N\times A}$.
\end{remark}

We replicate the results \cite[Section A.4]{Filipovic2022} for the vector-valued case, which is straightforward. For this we assume that $g$ is a vector-valued Gaussian process with mean function $m$ and kernel function $K(x,y)$, i.e.,
$g\sim\mathcal{MG}(m, K)$.
The pricing equation with errors is given by equation \eqref{eq:pricing_eq_error}
where we assume $\bm \epsilon\sim\mathcal{N}(0,\bm \Sigma)$ with
\[\bm{\Sigma}=
  \begin{pmatrix}
    \Sigma_1 & 0 & \dots  & 0\\
    0 &  \Sigma_2 & \dots   & 0 \\
    \vdots & \vdots  & \ddots &\vdots\\
    0 & \cdots & 0 & \Sigma_A
  \end{pmatrix} \in\mathbb{R}^{M\times M}\]
for symmetric positive definite $M_a\times M_a$-matrices $\Sigma_a$.

For $n$ arbitrary cash flow dates $\bm z=(z_1,\dots,z_n)^\top$ this implies that $\vect(g^\top(\bm z))$ and $\bm P$ are
jointly Gaussian distributed
\begin{equation}\label{eq:joint_distr}
  \begin{pmatrix}
    \vect(g^\top(\bm z))  \\
    \bm P
  \end{pmatrix}
  \sim\mathcal{N}\left(
    \begin{pmatrix}
      \vect(m^\top (\bm z))\\
      \bm C \vect(m^\top (\bm x))
    \end{pmatrix},
    \begin{pmatrix}
      K(\bm z, \bm z^\top)& K(\bm z, \bm{x}^\top)\bm C^\top\\
      \bm C K(\bm{x}, \bm z^\top) & \bm C\bm K \bm C^\top + \bm\Sigma
    \end{pmatrix}\right)
\end{equation}
where $K(\bm{x}, \bm z^\top)$ is the block matrix with entries $K(x_i,z_j)$, similar for $K(\bm z, \bm z^\top)$
such that $\bm K = K(\bm x, \bm x^\top)$.

Bayesian updating implies that the conditional distribution of $g$, given the observed prices $\bm P$, is still vector-valued Gaussian with posterior mean function
\begin{equation}
  \label{eq:post_mean}
  m^{\text{post}}(z) = m(z)+K(z,\bm x^\top)\vect(\beta^\top),
\end{equation}
with
\begin{equation}
\label{eq:mean_sol}
  \vect(\beta^\top) = \bm C^\top(\bm C \bm K \bm C^\top +\bm\Sigma)^{-1}(\bm P-\bm C \vect(m^\top(\bm x)),
\end{equation}
and posterior kernel function
\begin{equation*}
  K^{\text{post}}(y, z) = K(y, z)-K(y, \bm{x}^\top)\bm C^\top(\bm C \bm K\bm C^\top +\bm\Sigma)^{-1}\bm C K(\bm{x},z).
\end{equation*}
Hence we recovered the following vector-valued version of \cite[Lemma 9]{Filipovic2022}.

\begin{theorem}
  \label{thm:vect_post_gaussian_krr}
  Suppose the kernel $K$, the prior mean function $m=p$ and $\bm\Sigma = \bm\Lambda$ are as in Theorem \ref{thm:vect_krr}. Then
  the posterior mean function \eqref{eq:post_mean} coincides with the KR estimator $\bar g(z)$ in Theorem \ref{thm:vect_krr}.
\end{theorem}

The posterior mean is invariant with respect to scaling of $K$ and $\bm \Sigma$ by a factor $s>0$.
That is to replace $K$ by $K'=sK$ and $\bm\Sigma'=s\bm\Sigma$.
Similar as in \cite{Filipovic2022} one can use \eqref{eq:joint_distr} to derive at a prior log-likelihood function of $s$ 
given prices $\bm P$,
\begin{equation*}
  \mathcal{L}(s)=-q_2\frac{1}{s}-\frac{M}{2}\log{(s)} - q_1,
\end{equation*}
for $q_2 = \frac{1}{2}(\bm P- \bm C \vect(m^\top(\bm x))^\top(\bm C\bm K\bm C^\top +\bm \Sigma)^{-1}(\bm P- \bm C \vect(m^\top(\bm x)))$
and $q_1=\frac{1}{2}\log|\bm C\bm K\bm C^\top+\bm \Sigma| +\frac{M}{2}\log(2\pi)$.
The maximum log-likelihood is attained for
\begin{equation*}
  \hat{s} = \frac{2q_2}{M}.
\end{equation*}

\begin{remark}
When $\bm K_{\bm {ab}} = 0$ for all $a\neq b$ the posteriori mean estimator corresponds to $A>1$ independent scalar learned mean estimators. This can be seen from \eqref{eq:mean_sol} as in this case the block diagonal structure of $\bm K$ factors through, given that the matrices $\bm C$ and  $\bm \Sigma$ are blockdiagonal by definition. However, the confidence bands might differ as $\hat s$ does. In the scalar case the optimal scaling is given by $\hat s_a = \frac{2q_{2,a}}{M_a}$, for the respective value $q_{2,a}$. On the other hand, for the transfer learning case it holds by definition $M=\sum_a M_a$, and $\bm K_{\bm {ab}} = 0$ implies $q_2 = \sum_a q_{2,a}$. Hence in general the scaling factors differ, $\frac{q_{2,a}}{M_a}\neq \frac{\sum_b q_{2,b}}{\sum_b M_b}$, for individual classes $a$. 
\end{remark}

\subsection{Gaussian Process View for Separable Kernels}
The Gaussian process view reveals some additional interpretation for separable kernels, see Definition~\ref{def:separable_kernel}. In particular, one can use the theory of Gaussian matrix variate distributions to get some additional insights how different components of $g$ are correlated to each other. The key findings are given below. We first recall the definition and some basic properties of matrix variate
Gaussian distributions.
\begin{definition}
  The random matrix $X\in\mathbb{R}^{N\times A}$ is said to have a matrix variate Gaussian distribution with mean matrix
  $M\in\mathbb{R}^{N\times A}$, covariance matrices $\Sigma\in\mathbb{R}^{N\times N}$ and $B\in\mathbb{R}^{A\times A}$ if and only if
  the probability density function is given by
  \begin{equation*}
    p(X|M,\Sigma,B) = (2\pi)^{-\frac{AN}{2}}(\det{\Sigma})^{-\frac{A}{2}}(\det{B})^{-\frac{N}{2}}\exp{\left(-\frac{1}{2}\operatorname{tr}\left({B^{-1}(X-M)^\top\Sigma^{-1}(X-M)}\right)\right)}
  \end{equation*}
  We denote a matrix variate Gaussian distributed $X$ as $X\sim\mathcal{MN}(M,\Sigma,B)$.
\end{definition}

It holds that $X\sim\mathcal{MN}(M,\Sigma,B)$ if and only if $\vect(X)\sim\mathcal{N}(\vect(M),B\otimes \Sigma)$, see \cite[Theorem 2]{Chen2019}. This again implies that the transpose $X^\top\sim\mathcal{MN}(M^\top,B,\Sigma)$, see \cite[Theorem 1]{Chen2019}. Hence, in view of Definition \ref{def:vgp}, for a separable kernel $K(x,y)=B k(x,y)$, we have that $g\sim\mathcal{MG}(m, K)$ is equivalent to $g^\top(\bm x)\sim\mathcal{MN}(m^\top(\bm x),\bm k, B)$ for $\bm K=B\otimes \bm k$ and $m^\top(\bm x)=(m_1(\bm x),\dots, m_A(\bm x))$, and where $\bm k$ denotes the matrix with entries $\bm k_{ij}=k(x_i,x_j)$.

This leads to a natural interpretation of the variance and covariance structure of the discount curves. From the above, we obtain $\operatorname{Var}(g_a(x))=B_{aa}k(x,x)$ and $\operatorname{Cov}(g_a(x),g_b(y)) =B_{ab}k(x,y)$. The separable kernel structure allows us to interpret each entry \( B_{ab} \) as the covariance between product classes \( a \) and \( b \), scaled by the scalar kernel \( k(x,y) \), which reflects the maturity effect and is independent of the product class. The correlation is obtained by normalization. More details and a general decomposition result for matrix-valued kernels are provided in Lemmas~\ref{lemnormdec} and \ref{lem:normalizedsep} in the appendix.

\section{A Workable Class of Separable Kernels}
\label{sec:sep_kernels}

Overall, the framework developed in the previous two sections provides a direct extension of the single-curve setup in \cite{Filipovic2022} to multiple product classes. Formally, when setting $A=1$, the framework reduces exactly to that of \cite{Filipovic2022}.

In this section, we introduce the baseline model used in the empirical analysis below. The formulation is guided by economic reasoning following \cite{Filipovic2022} and leads to a tractable optimization problem with a closed-form solution. We proceed in two steps. First, we heuristically construct a joint estimation objective that incorporates spread penalties between curves. Second, we show that the resulting problem is equivalent to a vector-valued KR estimator with a separable matrix-valued kernel.

We begin with $A$ scalar-valued estimation problems, each for a fixed-income product class $a = 1, \dots, A$,
\[
  \min_{h_a\in\mathcal{H}_k} \sum_{i=1}^{M_a}\omega_{a,i}\big( P_{a,i}-C_{a,i}p_a(\bm x) -C_{a,i} h_a(\bm x)\big)^2 +  \gamma_a\|h_a\|^2_{\mathcal{H}_k},
\]
where $k$ is a common scalar kernel with RKHS $\Hcal_k$, and $\gamma_a > 0$ is the regularity hyperparameter for class $a$. Each problem yields an individual estimator $h_a$. Estimating the $A$ curves independently is equivalent to solving the joint optimization problem
\[
  \min_{h_1,\dots,h_A\in\mathcal{H}_k} \sum_{a=1}^A \left\{ \sum_{i=1}^{M_a}\omega_{a,i}( P_{a,i}-C_{a,i}p_a(\bm x) -C_{a,i} h_a(\bm x))^2 + \gamma_a\|h_a\|^2_{\mathcal{H}_k} \right\}.
\]
This can be viewed as a single objective over the product space $(\mathcal{H}_k)^A$.

To introduce dependencies across product classes, we extend the regularization to the differences between curves. Specifically, we add spread penalties of the form
\[
  \sum_{a=1}^A\sum_{b>a}\Theta_{ab}\|h_a - h_b\|^2_{\mathcal{H}_k},
\]
where $\Theta_{ab} \ge 0$ controls the strength of transfer learning between classes $a$ and $b$.\footnote{We also considered adjusting the individual regularization parameters $\gamma_a$ downward to keep the total regularization weight constant when adding spread penalties. This corresponds to choosing $\lambda < 1$ in \eqref{eq:krr_obj}. However, in our empirical studies we found that such scaling can introduce irregularities in the estimated discount curves, which is undesirable. We therefore recommend keeping the values of $\gamma_a$ fixed and setting $\lambda = 1$ to achieve a well-balanced and effective transfer learning outcome, as stated in Theorem~\ref{thm_vkr}.} These terms encourage similarity between curves without forcing equality. Instead, they penalize irregularities in the spread curves through the RKHS norm $\|\cdot\|_{\mathcal{H}_k}$. We use the terms regularity and smoothness interchangeably, referring specifically to the notion of smoothness induced by the RKHS norm $\|\cdot\|_{\mathcal{H}_k}$. The following example presents a kernel $k$ that encodes an economically meaningful notion of smoothness introduced in \cite{Filipovic2022}.

\begin{example}
\label{ex:kernel}
Consider the scalar kernel
\begin{equation} \label{eqkSobol}
  k(x,y) =  -\frac{\min\{x,y\}}{\alpha^2}\e^{-\alpha \min\{x,y\}}+\frac{2}{\alpha^3}\left( 1 - \e^{-\alpha\min\{x,y\}}\right)-\frac{\min\{x,y\}}{\alpha^2}\e^{-\alpha \max\{x,y\}},
\end{equation}
with maturity-weight hyperparameter $\alpha>0$.\footnote{The RKHS introduced in \cite{Filipovic2022} is more flexible, as its norm~\eqref{rkhs_norm} includes both first- and second-order derivatives. However, their empirical analysis on US\ data finds that only the second-order term is relevant for the performance of the KR estimator. \cite{CamenzindFilipovic2024} confirm this finding for Swiss data as well.} \cite{Filipovic2022} show that the corresponding RKHS $\mathcal{H}_k$ is a weighted Sobolev space consisting of twice weakly differentiable functions $h:[0,\infty)\to\mathbb{R}$ with $h(0)=0$, $\lim_{x\to\infty} h'(x)=0$, and finite smoothness norm given by
\begin{equation}
\label{rkhs_norm}
      \|h\|_{\mathcal{H}_k}^2 =  \int_0^\infty h''(x)^2 \e^{\alpha x}\,dx .
\end{equation}
\end{example}

The complete transfer learning problem is
\begin{equation}
  \label{eq:multi_curve_motivated}
  \min_{h_1,\dots,h_A\in\mathcal{H}_k}\sum_{a=1}^A\Bigl\{\sum_{i=1}^{M_a}\omega_{a,i}( P_{a,i}-C_{a,i}p_a(\bm x) -C_{a,i} h_a(\bm x))^2 + \gamma_a\|h_a\|^2_{\mathcal{H}_k}+\sum_{ b>a}\Theta_{ab}\|h_a-h_b\|_{\Hcal_k}^2\Bigr\}.
\end{equation}
The following theorem shows that \eqref{eq:multi_curve_motivated} can be interpreted as a vector-valued KR with a separable kernel. This implies in particular that the optimization problem is convex and admits a unique solution.

\begin{theorem}
\label{thm_vkr}
Let $\Theta_{ab} \ge 0$ for $a < b$, and define $\Theta_{ba} = \Theta_{ab}$. Then the transfer learning problem \eqref{eq:multi_curve_motivated} is equivalent to the vector-valued KR problem \eqref{eq:krr_obj} with $\lambda=1$ and vector-valued RKHS norm
\begin{equation}
\label{eq:graph_norm}
\|h\|^2_{\mathcal{H}} = \sum_{a=1}^A \gamma_a \|h_a\|^2_{\mathcal{H}_k} + \sum_{a=1}^A \sum_{b>a} \Theta_{ab} \|h_a - h_b\|^2_{\mathcal{H}_k}
\end{equation}
which corresponds to the separable reproducing kernel $K(x,y) = B k(x,y)$, where $B = Q^{-1}$ and $Q \in \mathbb{R}^{A \times A}$ is defined by
\begin{equation}
\label{eqgraphreg}
Q_{ab} = \begin{cases}
\gamma_a + \sum_{j \neq a} \Theta_{aj}, & \text{if } a = b, \\
- \Theta_{ab}, & \text{if } a \ne b.
\end{cases}
\end{equation}
\end{theorem}

The hyperparameters $\Theta_{ab}$ may be interpreted as edge weights on a graph with $A$ nodes, each corresponding to a product class. The matrix $Q$ equals the sum of the diagonal matrix of $\gamma_a$ and the Laplacian of the graph, $Q = \mathrm{diag}(\gamma) + L(\Theta)$. This formulation is known as graph regularization in the literature, see \cite{baldassarre:multi_output_learning} and \cite{sheldon:graph_multi_task_learning}.

In sum, in conjunction with the scalar kernel \eqref{eqkSobol}, this specification includes the following hyperparameters: $\alpha$ (scalar kernel parameter), $\gamma_a$ (discount curve smoothness), $\Theta_{ab}$ (spread smoothness).

\section{Standard Fixed-Income Products}
\label{sec:std_product_formulation}
This section shows how standard fixed-income instruments can be expressed in the discounted cash flow format~\eqref{eq:pricing_eq}, thereby enabling application of our estimation framework. While this formulation may lead to distinct discount curves across different product classes, we demonstrate in Appendix~\ref{appendix:na} how, under an arbitrage-free pricing framework, a single risk-free curve and the corresponding function~$g$ may be recovered.

We proceed as follows: we first show how coupon bonds can be cast into the pricing format~\eqref{eq:pricing_eq}, then extend this formulation to fixed--floating interest-rate swaps. Finally, we examine cross-currency swaps and show how transfer learning facilitates joint estimation of discount curves and forward exchange rates, offering insights into multi-currency pricing.

\subsection{Coupon Bonds}

The transformation of fixed-coupon bonds into the discounted cash flow format~\eqref{eq:pricing_eq} is straightforward. Consider a bond with notional normalized to one and coupons $c_1,\dots,c_n$ paid at dates $0 < T_1 < \dots < T_n$, where $T_n$ denotes the bond's maturity at which the notional is paid.\footnote{The generic time grid $(x_i)$ used in~\eqref{eq:pricing_eq} is assumed to be fine enough to cover all potential cash flow dates across product classes. Hence, most entries in each row of $\bm{C}$ are zero.} Assuming the bond is default-free, the price is given by
\[
P = \sum_{j=1}^n c_j g(T_j) + g(T_n).
\]
Defaultable bonds are treated in Appendix~\ref{appendix:na} under an arbitrage-free pricing framework.

\subsection{Interest-Rate Swaps}\label{sec:swaps}
We consider a standard fixed--floating interest-rate swap based on the risk-free rate (RFR) with start (first reset) date $T_0\ge 0$ and maturity date $T_n$. We denote the reset and cash flow dates of the fixed payments leg by $T_0<T_1<\dots<T_n$ and of the floating payments leg by $T_0 = t_0<t_1<\dots<t_m=T_n$. For simplicity, the accrual periods along both legs are assumed to be constant and denoted by $\Delta = T_i-T_{i-1}$ and $\delta=t_i-t_{i-1}$, respectively.\footnote{This can be generalized to specific day count conventions for both legs where the accrual periods depend on the actual dates, replacing the constant $\Delta$ and $\delta$ by $\Delta(T_{i-1},T_i$) and $\delta(t_{j-1},t_j)$, respectively.}
The swap is spot starting when
$T_0=0$ and forward starting when $T_0>0$. 

The present values of the fixed and floating legs are given by
\begin{align*}
  PV_{\text{fixed}} &= \Delta R\sum_{i=1}^n g(T_i) ,\\
  PV_{\text{floating}} &= g(T_0) - g(T_n),
\end{align*}
where $R$ denotes the corresponding fixed swap rate. The derivation follows from standard no-arbitrage arguments and is provided in Appendix~\ref{appendix:na} for completeness. At inception, the swap has zero value, so that $PV_{\text{floating}}=PV_{\text{fixed}}$. We bring this into the desired format \eqref{eq:pricing_eq} as follows. For a spot-starting swap, $T_0=0$, the price is set to $P=1$, which gives
\begin{equation} 
  \label{sswapcf}
  1 = g(T_n) +\Delta R\sum_{i=1}^n g(T_i).
\end{equation}
For a forward-starting swap, $T_0>0$, the price is set to $P=0$, which gives
\begin{equation} 
  \label{fswapcf}
  0 = g(T_n)-g(T_0) +\Delta R\sum_{i=1}^n g(T_i).
\end{equation}

Based on \eqref{sswapcf} and \eqref{fswapcf}, the YTM $Y$ of the swap can be derived as defined in Example~\ref{ex:weights}. The following result links the YTM $Y$ to the swap rate $R$.
\begin{lemma}
\label{lemYR}
If $T_j - T_{j-1} \equiv \Delta$ for all $j=1, \dots, n$, then $\Delta Y = \log(1 + \Delta R)$. That is, in first order the YTM equals the swap rate, $Y \approx R$.
\end{lemma} 

The following example illustrates this for a single-period overnight swap.
\begin{example}
In the US, the overnight RFR is SOFR, here denoted by $R_{\text{SOFR}}$. Consider a single-period overnight swap maturing at $T_1 = \frac{1}{365}$. In view of \eqref{sswapcf}, its price as a function of YTM~$Y$ is given by $\Pi_{\text{SOFR}}(Y)=e^{-YT_1}(1+T_1R_{\text{SOFR}})$. The market-implied YTM $Y_{\text{SOFR}}$ is defined by $\Pi_{\text{SOFR}}(Y_{\text{SOFR}}) = 1$, which implies that
\[
Y_{\text{SOFR}} = \frac{1}{T_1} \log(1 + T_1 R_{\text{SOFR}}) \approx R_{\text{SOFR}},
\]
is equal to the SOFR up to first order. The derivative $\Pi'_{\text{SOFR}}(Y_{\text{SOFR}}) = -T_1$ then gives the corresponding duration-based weight in Example~\ref{ex:weights}.
\end{example}

\subsection{Cross-Currency Swaps}
Cross-currency swaps (XCCY) involve cash flows in two currencies and combine features of both interest rate and foreign exchange (FX) instruments. See, e.g., \cite{Ranaldo2022XCCY,ECBXCCY} for introductions. We denote the spot exchange rate prevailing at time $t$ as $X_{ab}(t)$, defined as the price of one unit of (base) currency $a$ in terms of (quote) currency $b$. 

A typical use case involves a domestic, say Swiss, firm that holds CHF and wishes to buy a USD-denominated bond. To hedge currency risk, the firm enters a XCCY swapping USD coupon payments against CHF cash flows.\footnote{Another example is two firms located in different countries with different currencies. Each exhibits cheaper local funding sources. To raise funds abroad they can enter into a bilateral XCCY.}

The most standardized and actively traded XCCY is the floating--floating type used in the interbank market. At the start date $t_0$, notional amounts in both currencies are exchanged at the prevailing spot exchange rate $X_{ab}(t_0)$. Thereafter, floating interest payments are made in each currency at dates $ t_0<t_1<\dots<t_m$, typically quarterly and based on RFRs. The initial notional amounts are re-exchanged at maturity date $t_m$. A basis spread $s$ is typically added to the less liquid currency leg to reflect liquidity differences and funding imbalances between the two currencies. Figure~\ref{fig:xccy_cf} illustrates this.
 
\begin{figure}[h!]
\centering
\tcapfig{Schematic cash flows of a floating--floating XCCY swap}
\begin{tikzpicture}[
    axis/.style={very thick},
    arrow/.style={-Latex, thick},
    floatwiggle/.style={decorate, decoration={zigzag, amplitude=0.5mm, segment length=2mm}, thick},
    >=Latex,
    every node/.style={font=\small}
]

\draw[axis] (-0.5,0) -- (11,0) node[right] {$t$};

\foreach \x/\label in {
    1.5/t_0, 
    3.5/t_1, 
    5.5/t_2,
    7.5/\dots,
    9.5/t_m}
    \draw (\x,0.1) -- (\x,-0.1)
          node[anchor=north east, xshift=-2pt, yshift=-2pt] {$\label$};

\node at (0.5,2.2) {Receive floating in leg $a$};
\node at (0.5,-2) {Pay floating in leg $b$};

\draw[->, thick, blue] (1.5,0) -- ++(0,2.0) node[midway, left] {$X_{ab}(t_0)$};  
\draw[->, thick, red] (1.5,0) -- ++(0,-1.5) node[midway, right] {$1$};  

\draw[floatwiggle, blue] (3.5,0) -- ++(0,-1.0);
\draw[->, thick, blue] (3.5,-1.0) -- ++(0,-0.2) node[midway, right] {};
\draw[->, thick, red] (3.7,0) -- ++(0,-0.5) node[near end, right] {$s$};

\draw[floatwiggle, red] (3.5,0) -- ++(0,1.3);
\draw[->, thick, red] (3.5,1.3) -- ++(0,0.2) node[midway, left] {};

\draw[floatwiggle, blue] (5.5,0) -- ++(0,-1.3);
\draw[->, thick, blue] (5.5,-1.3) -- ++(0,-0.2) node[midway, right] {};

\draw[floatwiggle, red] (5.5,0) -- ++(0,1.5);
\draw[->, thick, red] (5.5,1.5) -- ++(0,0.3) node[midway, left] {};
\draw[->, thick, red] (5.7,0) -- ++(0,-0.5) node[near end, right] {$s$};

\draw[floatwiggle, blue] (9.5,0) -- ++(0,-0.8);
\draw[->, thick, blue] (9.5,-0.8) -- ++(0,-0.2) node[midway, right] {};

\draw[floatwiggle, red] (9.5,0) -- ++(0,0.5);
\draw[->, thick, red] (9.5,0.5) -- ++(0,0.3) node[midway, left] {};

\draw[->, thick, red] (9.7,0) -- ++(0,-0.5) node[near end, right] {$s$};

\draw[->, thick, red] (10.1,0) -- ++(0,1.5) node[near end, right] {$1$};  
\draw[->, thick, blue] (10.1,0) -- ++(0,-2.0) node[near end, right] {$X_{ab}(t_0)$};  

\end{tikzpicture}
\label{fig:xccy_cf}
\bnotefig{The figure shows the cash flow diagram of a floating--floating XCCY swap. We take the view of receiving leg $a$ while making periodically payments in leg $b$. Thus, downward pointed arrows reflect a cash flow we have to pay. The wiggled lines denote the floating payments. Straight line are the exchange of notionals and basis spread payments. We assume the basis spread $s$ is added on leg $a$. It is common that this spread is negative, indicated by the downward pointed arrow. We use the convention to normalize the notional of leg $a$ to $1$ so that the corresponding notional of leg $b$ is given by $X_{ab}(t_0)$. }
\end{figure}

An additional feature common in interbank markets is mark-to-market (MTM) resets of the notional leg. These reduce counterparty risk but, as shown in Appendix~\ref{appendix:na}, have no impact on present values. 

End clients generally prefer fixed interest payments. Banks accommodate this by combining floating--floating XCCY with standard fixed--floating interest-rate swaps. This composite structure is also
necessary for estimating the discount curve using our framework. We focus on the non-liquid leg (currency $a$), and bring this now into the desired format \eqref{eq:pricing_eq}. Thereto, let $R_a$ be the fixed swap rate of a standard RFR-based swap in currency~$a$ with the same maturity as the XCCY. We assume this rate is observable from the market.\footnote{This is standard practice in well-developed swap markets.} More specifically, let $t_0=T_0 < T_1<\dots < T_n = t_m$ be the fixed leg's payment dates with constant accrual period $\Delta=T_i-T_{i-1}$, and assume that currency~$b$ is the more liquid leg, so the basis spread $s$ is added to leg~$a$. For a spot-starting XCCY, $T_0=0$, we set $P = 1$. The corresponding discounted cash flow equation becomes
\[
1 = g_{a:b}(T_n) + \Delta R_a \sum_{i=1}^n g_{a:b}(T_i) + \delta s \sum_{j=1}^m g_{a:b}(t_j).
\]
For a forward-starting XCCY, $T_0 > 0$, the price is set $P = 0$ and we obtain
\[
0 = g_{a:b}(T_n) - g_{a:b}(T_0) + \Delta R_a \sum_{i=1}^n g_{a:b}(T_i) + \delta s \sum_{j=1}^m g_{a:b}(t_j).
\]
Here, $g_{a:b}(\cdot)$ denotes the discount curve for currency~$a$ induced by currency~$b$ via an XCCY. It incorporates the cross-currency basis and is generally distinct from the discount curve $g_a(\cdot)$ that corresponds to standard interest-rate swaps in currency $a$. If the basis spread $s$ is zero, $g_{a:b}(\cdot)$ coincides with $g_a(\cdot)$.

An important byproduct of this formulation is an expression for the forward exchange rate that incorporates the cross-currency basis. Let $F_{ab}(x)$ denote the forward exchange rate fixed at time~$0$ for maturity $x$. Then
\begin{equation}\label{FXgg}
F_{ab}(x) = X_{ab}(0)\frac{g_{a:b}(x)}{g_b(x)}.
\end{equation}
This identity can be derived by considering a spot-starting XCCY in combination with an interest-rate swap with a single payment at $t_1 = T_1 = x$. Investing one unit of currency~$a$ via this XCCY and swapping the floating payment for fixed results in a fixed payoff of $\frac{1}{g_{a:b}(x)}$ units of currency~$a$ at maturity $x$. Alternatively, the same initial amount can be used to purchase $\frac{X_{ab}(0)}{g_b(x)}$ units of the discount bond with maturity $x$ in currency~$b$. At maturity, this yields a cash flow in currency~$b$, which is then converted back into currency~$a$ at the forward exchange rate $F_{ab}(x)$, resulting in a payoff of $\frac{X_{ab}(0)}{g_b(x) F_{ab}(x)}$ in currency~$a$. Since both strategies yield deterministic payoffs, the absence of arbitrage implies that they must be equal, which proves~\eqref{FXgg}.

\begin{remark}
Textbook covered interest parity (CIP) posits that $F_{ab}(x) = X_{ab}(0)\frac{g_a(x)}{g_b(x)}$. However, in practice, deviations from CIP are persistent, as $g_{a:b}(x) \ne g_a(x)$ due to liquidity differences and funding constraints. Most currencies exhibit a negative basis against USD, meaning counterparties are willing to accept a lower return to obtain USD funding, which manifests as $s < 0$ in observed XCCY swaps.
\end{remark}

\section{Empirical Analysis}
\label{sec:empirical_illustrations}

This section assesses the economic significance of transfer learning. We begin by describing the data used in the empirical analysis and the associated hyperparameter selection. We then illustrate the qualitative effects of transfer learning on the estimated yield and forward rate curves to build intuition for its impact, and examine the implications for estimation uncertainty from a Gaussian process perspective. Finally, we conduct an extensive masking experiment to evaluate the benefits of transfer learning in an out-of-sample setting. 

\subsection{Data and Hyperparameter Selection}
\label{sec:data_hyperparm_sel}

We use daily data on US government bonds and SOFR swaps obtained from Bloomberg Finance L.P. The sample spans January 2020 through December 2024. 

For US government bonds, we rely on daily end-of-day (mid) dirty prices and assume same-day settlement. The sample includes all fully taxable, non-callable coupon bonds and excludes Treasury bills. In contrast to \cite{gur_sac_wri_07}, we do not impose additional filters such as excluding short-maturity bonds. This choice allows our estimation framework to demonstrate robustness across the full observable cross section of government bonds.

Figure~\ref{fig:max_mat} illustrates the maturity distribution of US government bonds over the sample period.\footnote{A bond enters the sample as soon as it has at least one valid price observation, i.e., a non-NA (non-missing value). For clarity, missing observations on specific days are not displayed in the figure.} The dataset comprises 610 bonds with well dispersed maturities. The longest outstanding bonds have 30-year maturities and are issued regularly throughout the sample period. Although Treasury bills are excluded, shorter-dated coupon bonds still constitute a meaningful share of the sample, ensuring balanced coverage across the maturity spectrum.

\begin{figure}[h!]
  \centering
     \tcapfig{Maturities of US government bonds}
  \begin{subfigure}[t]{0.49\linewidth}
  \centering
    \includegraphics[width=\linewidth]{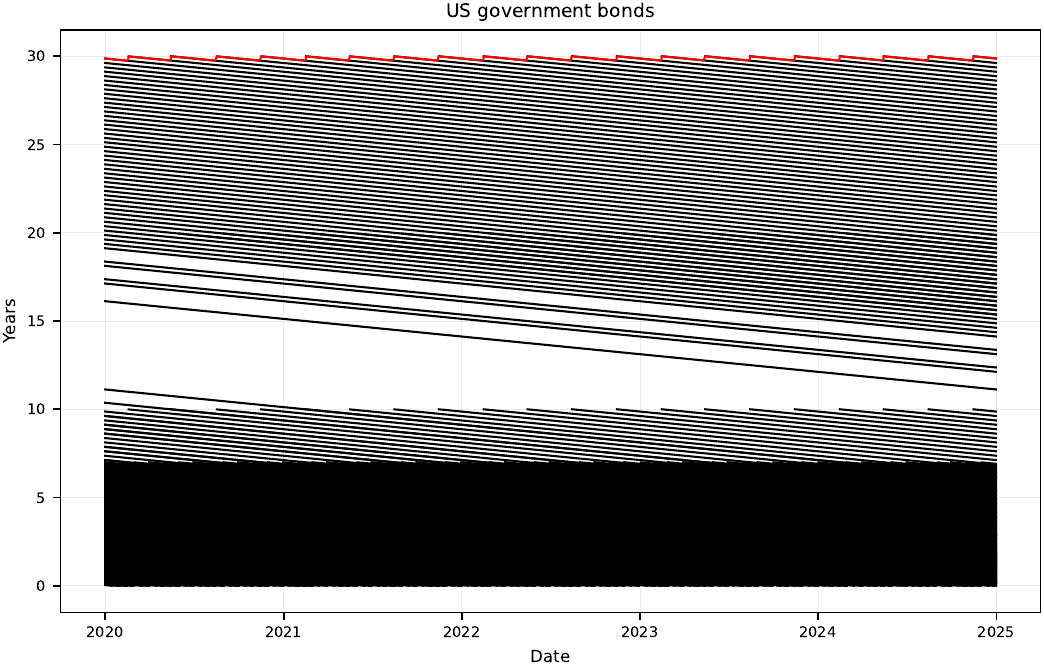}
  \end{subfigure}
  \bnotefig{The figure plots the available bonds and their respective maturities over the sample period. The red line marks the longest maturity among outstanding bonds at each point in time.}
 \label{fig:max_mat}
\end{figure}

For SOFR swaps, we use daily end-of-day (mid) swap rates. The floating leg of these interest-rate swaps references the Secured Overnight Financing Rate (SOFR) \cite{FedSofr}, which is a volume-weighted average of fully collateralized overnight repurchase (repo) transactions. SOFR has been published by the Federal Reserve Bank of New York since April~2018.

The available SOFR swap tenors are
1D, 1W, 2W, 3W, 1M, 2M, 3M, 4M, 5M, 6M, 7M, 8M, 9M, 10M, 11M, 1Y, 13M, 14M, 15M, 16M, 17M, 18M, 19M, 20M, 21M, 22M, 23M, 2Y, 27M, 30M, 33M, 3Y, 42M, 4Y, 54M, 5Y, 6Y, 7Y, 8Y, 9Y, 10Y, 11Y, 12Y, 15Y, 20Y, 25Y, 30Y, 35Y, 40Y, 45Y, and 50Y.
This maturity structure spans a wide range, from overnight to 50~years, with particularly dense coverage at the short end. Such granularity makes SOFR swaps a natural benchmark for transfer learning and a valuable source of information for extrapolating the US government bond discount curve up to 50Y.

For model evaluation, we use the root mean squared error (RMSE) of the YTM $Y$ for bonds (and, analogously, the swap rate $R$ for swaps) at time~$t$, defined as $\text{RMSE}_t \coloneqq \sqrt{\frac{1}{M_t}\sum_{i=1}^{M_t}\big(Y_{i,t}-\hat{Y}^g_{i,t}\big)^2}$, where $M_t$ denotes the number of instruments observed at time~$t$. Overall performance is summarized by the time-averaged $\text{RMSE} \coloneqq \frac{1}{T}\sum_{t=1}^T \text{RMSE}_t$. Results are reported both in aggregate and by maturity buckets: $\leq 1\text{Y}$, $(1\text{Y},5\text{Y}]$, $(5\text{Y},10\text{Y}]$, $(10\text{Y},15\text{Y}]$, $(15\text{Y},20\text{Y}]$, $(20\text{Y},25\text{Y}]$, $(25\text{Y},30\text{Y}]$, $(30\text{Y},35\text{Y}]$, $(35\text{Y},40\text{Y}]$, $(40\text{Y},45\text{Y}]$, and $(45\text{Y},50\text{Y}]$. In what follows, we use the terms fitting error and RMSE interchangeably. 

Figure~\ref{fig:bid_ask} displays the distribution of Bid--Ask spreads for US government bonds (left panel) and SOFR swaps (right panel) across maturity buckets.\footnote{The data use daily end-of-day bid and ask yields and swap rates sourced from Bloomberg Finance L.P.} For bonds, a logarithmic scale is applied to the shortest maturity bucket ($\leq 1$Y), while a linear scale is used for all remaining buckets. For swaps, all maturity buckets are displayed on a linear scale. Overall, both product classes exhibit tight Bid--Ask spreads, typically below 2~basis points (bps). The widest spread distribution is observed for bonds in the $\leq 1$Y maturity bucket.

\begin{figure}[ht!]
     \tcapfig{Bid--Ask yield and swap spreads}
     \begin{subfigure}[t]{0.49\linewidth}
        \centering
        \includegraphics[width=\linewidth]{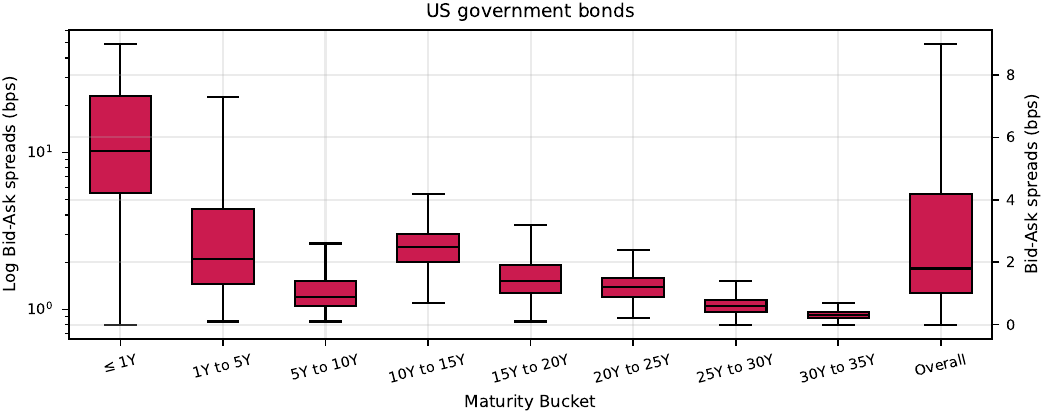}
     \end{subfigure}\hfill
     \begin{subfigure}[t]{0.49\linewidth}
        \centering
        \includegraphics[width=\linewidth]{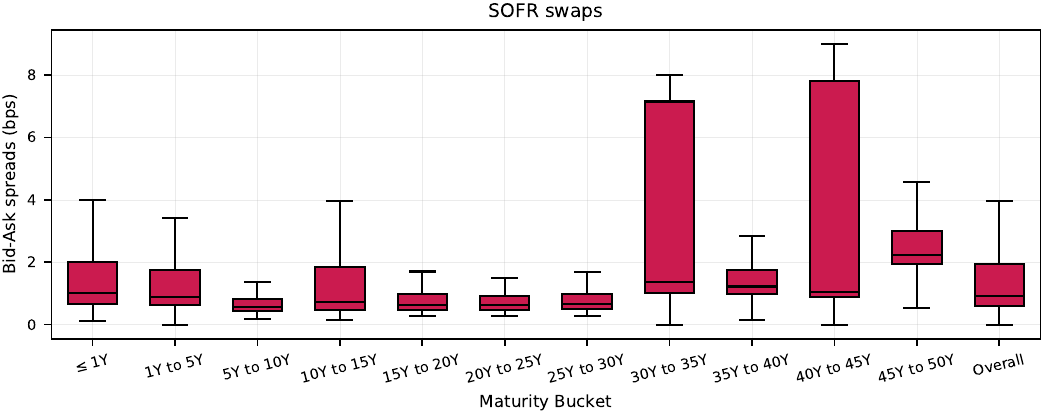}
     \end{subfigure}\vspace{0.3em}
     \bnotefig{Distribution of Bid--Ask yield spreads by maturity bucket for US~government bonds (left) and SOFR swaps (right). The box plots display interquartile ranges in red, with the rightmost bucket indicating the overall aggregate. All values are in bps.}
 \label{fig:bid_ask}
\end{figure}

Table~\ref{tab:bid_ask_stat} summarizes average Bid--Ask spreads for bond yields and swap rates. For US government bonds, we report two measures: the overall average and the average computed after excluding the shortest maturity bucket ($\leq 1$Y). Short-dated bonds exhibit disproportionately large yield movements in response to small price changes, which mechanically inflates observed Bid--Ask spreads. Excluding these maturities substantially lowers the average spread, while the median remains largely unchanged. These magnitudes provide natural benchmarks for assessing the economic relevance of our empirical results: improvements on the order of 2bps are economically negligible, whereas larger differences are economically significant.

\begin{table}[htbp]
  \centering
  \begin{tabular}{lccc}
    \toprule
     & \multicolumn{2}{c}{\textbf{US~government bonds}} & \textbf{SOFR swaps} \\
     & Overall & Excluding $\leq1$Y & Overall \\
    \midrule
    Average & 10.03 & 2.54 & 1.74 \\
    Median  & 1.80  & 1.50 & 0.91 \\
    \bottomrule
  \end{tabular}
  \caption{Average and median Bid--Ask spreads for US government bonds and SOFR swaps. All values are in bps.}
  \label{tab:bid_ask_stat}
\end{table}

All empirical results reported below are based on the estimation setup introduced in Section~\ref{sec:sep_kernels} with $A=2$ product classes. Throughout, we employ the duration-based weighting scheme from Example~\ref{ex:weights} in combination with the scalar kernel $k(x,y)$ specified in Example~\ref{ex:kernel}. As prior we use the constant function $p=1$. Consistent with the modular setup outlined in Section~\ref{sec:sep_kernels}, we proceed sequentially. We first select the kernel parameter $\alpha$ together with the standalone discount curve smoothness parameters $\gamma_a$. Given these choices, we then select the remaining spread smoothness parameter $\theta = \Theta_{12}$ through a masking experiment. 

Following \cite{Filipovic2022,CamenzindFilipovic2024}, selection of the standalone hyperparameters $\alpha$ and $\gamma_a$ is carried out using a daily leave-one-out cross-validation (LOOCV) procedure applied separately to bonds and swaps. We consider the following parameter grids: $\alpha \in \{0.01, 0.02, 0.03, 0.04, 0.05, 0.06, 0.07, 0.08, 0.09, 0.10\}$ and $\gamma \in \{0.01, 0.05, 0.1, 0.5, 1, 5, 10, 50, 100\}$. For notational convenience, values of $\gamma$ are scaled by $10^{-4}$, so that a reported value of~10 corresponds to $10^{-3}$ in the implementation, consistent with the convention adopted in \cite{CamenzindFilipovic2024}.

Figure~\ref{fig:heatmap} displays the aggregated RMSE in bps across hyperparameter combinations $(\gamma,\alpha)$. These heatmaps underscore the robustness of the method to hyperparameter variation, a desirable property that is consistent with the findings in previous literature. The RMSE-minimizing hyperparameters are $(\gamma,\alpha) = (0.1,0.01)$ for bonds and $(\gamma,\alpha) = (0.5,0.06)$ for swaps.\footnote{For US government bonds, \cite{Filipovic2022} report optimal standalone hyperparameters of $(\gamma,\alpha) = (1,0.05)$ based on a different sample and time period. They also employ a different scaling of $\gamma$, namely $1/(365\cdot x_N)$ for the longest available maturity $x_N$ per cross-section.} In light of the observed robustness, and to align with previous literature \cite{Filipovic2022}, we adopt the common choice $\alpha = 0.05$ and $\gamma = 1$ for both US government bonds and SOFR swaps. This parsimonious specification reduces tuning complexity without sacrificing empirical performance.

\begin{figure}[ht!]
     \tcapfig{LOOCV RMSE heatmaps}
     \begin{subfigure}[t]{0.49\linewidth}
        \centering
        \includegraphics[width=\linewidth]{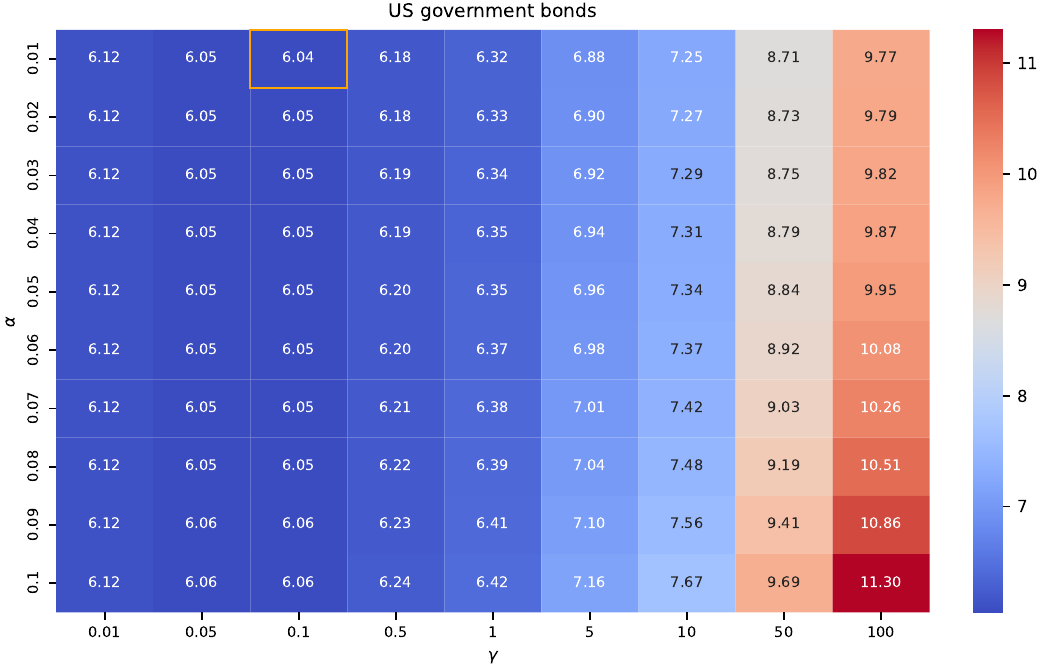}
     \end{subfigure}\hfill
     \begin{subfigure}[t]{0.49\linewidth}
        \centering
        \includegraphics[width=\linewidth]{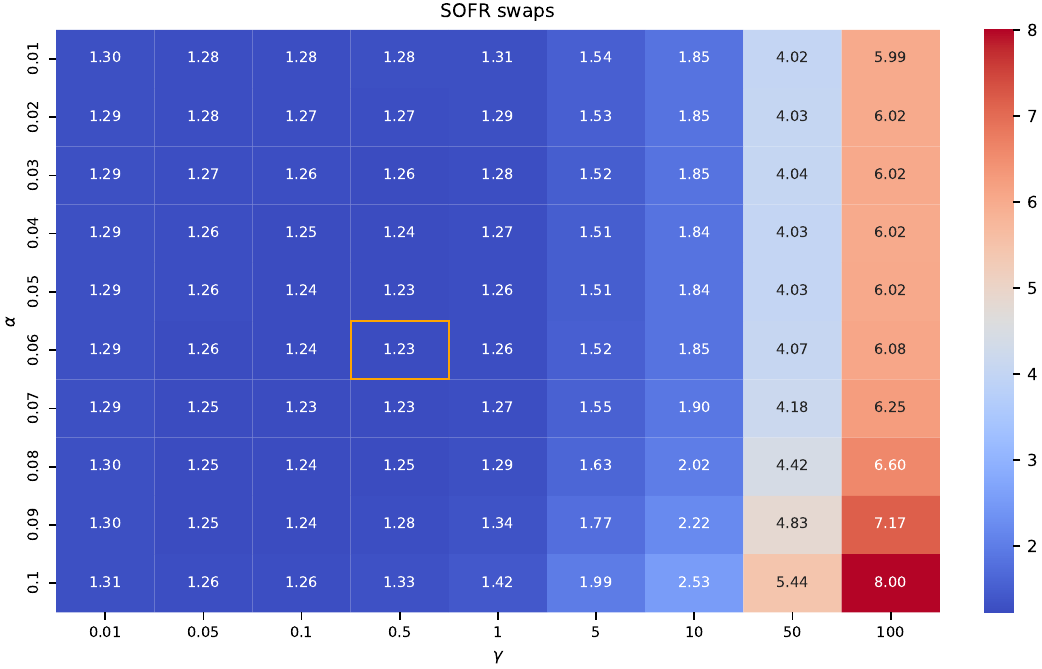}
     \end{subfigure}\vspace{0.3em}
     \bnotefig{Time-averaged YTM RMSE (left) and swap rate RMSE (right) for US government bonds and SOFR swaps. Dark blue areas denote low RMSE, red areas high RMSE, and the orange rectangle highlights optimal hyperparameter choices. All values are in bps.}
 \label{fig:heatmap}
\end{figure}

\subsection{Illustrative Yield and Forward Curves}
\label{sec:ill_yield_and_fwd_curves}

Before turning to the masking experiment used to select $\theta$, we first present illustrative examples to build intuition for its effects. To this end, given the standalone hyperparameters selected above, we apply transfer learning and jointly estimate the discount curves for bonds and swaps for a range of values of $\theta$.

Following \cite{CamenzindFilipovic2024}, we focus on the mid-June (nearest available) business day of each year in the sample. Figure~\ref{fig:example_day_2024} reports the resulting yield curves (left panel) and forward rate curves (right panel) for the most recent example day, 2024-06-14. Across rows, the transfer learning parameter $\theta$ increases from 0 (no transfer learning) to 10, 100, and 1000, illustrating its main effects.\footnote{Intermediate values of $\theta$ yield qualitatively similar results and are omitted for brevity.} For the yield curves, we also report the Gaussian process interpretation of the curve estimates: shaded areas represent $\pm 3\sigma$ confidence bands, truncated at $\pm 2\%$ for readability. Vertical dashed lines indicate the longest available maturities for bonds and swaps on the given date. As expected, US government bonds extend to 30~years, while SOFR swaps cover maturities up to 50~years.

The effect of $\theta$ on the yield curves is most clearly reflected in the confidence bands. When $\theta=0$, estimation uncertainty for the bond curve increases sharply beyond 30~years, whereas swap confidence bands remain tight across the entire maturity range. Within the maturity region supported by bond data, uncertainty remains low. As $\theta$ increases, uncertainty in the bond extrapolation region declines markedly, with the most pronounced reduction occurring between $\theta=0$ and $\theta=100$. In contrast, the swap confidence bands are largely unaffected. A similar pattern is visible in the yield levels themselves: the impact of transfer learning is concentrated in the extrapolation region, where the 50-year bond yield increases by roughly 50bps for large values of $\theta$. Within the well-observed maturity range, yield curves are nearly indistinguishable across values of $\theta$. This behavior is desirable, as transfer learning stabilizes data-sparse regions without distorting well-identified segments of the curve.

The right panel of Figure~\ref{fig:example_day_2024} displays the corresponding forward rate curves, where the effects of transfer learning are even more pronounced. For large values of $\theta$, the swap forward curve becomes noticeably more irregular, reflecting spillovers from the less smooth bond forward curve. This indicates that excessively large values of $\theta$ can induce undesirable bidirectional information transfer, underscoring the importance of moderate calibration.

\begin{figure}[h!]
  \centering
       \tcapfig{Example day 2024-06-14}
  \begin{subfigure}[t]{0.49\linewidth}
    \centering
    \includegraphics[width=\linewidth]{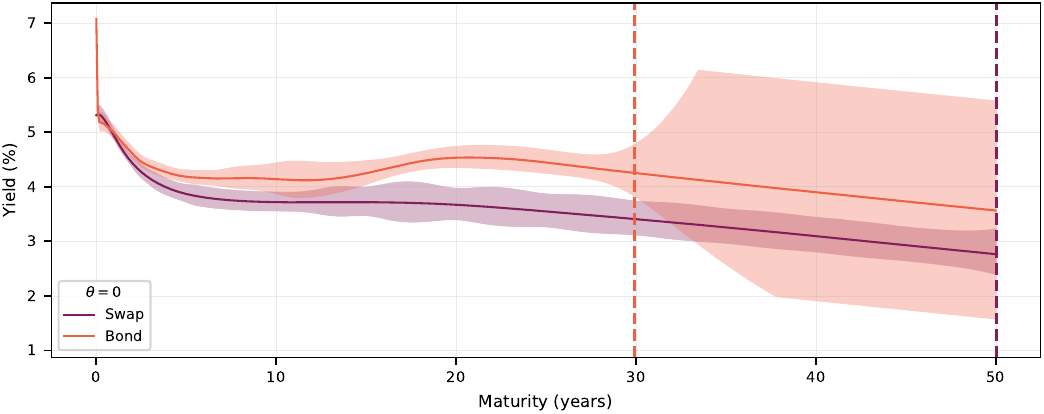}
  \end{subfigure}\hfill
  \begin{subfigure}[t]{0.49\linewidth}
    \centering
    \includegraphics[width=\linewidth]{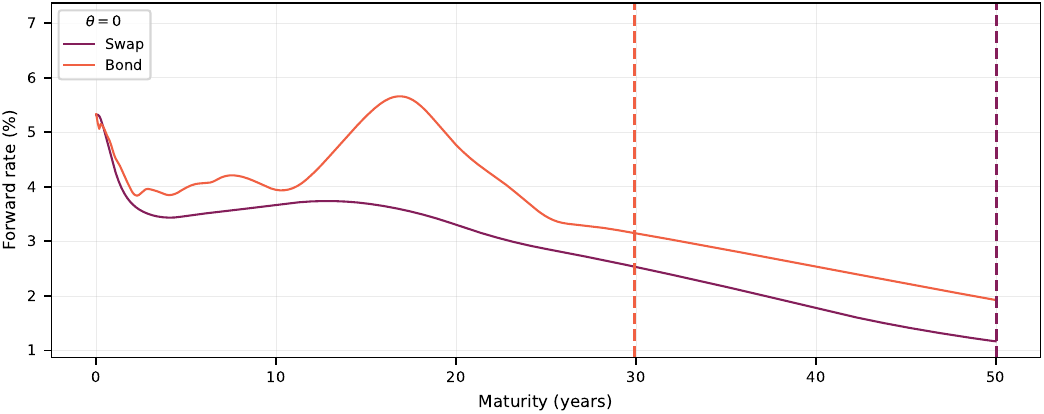}
  \end{subfigure}\vspace{0.3em}
    \begin{subfigure}[t]{0.49\linewidth}
    \centering
    \includegraphics[width=\linewidth]{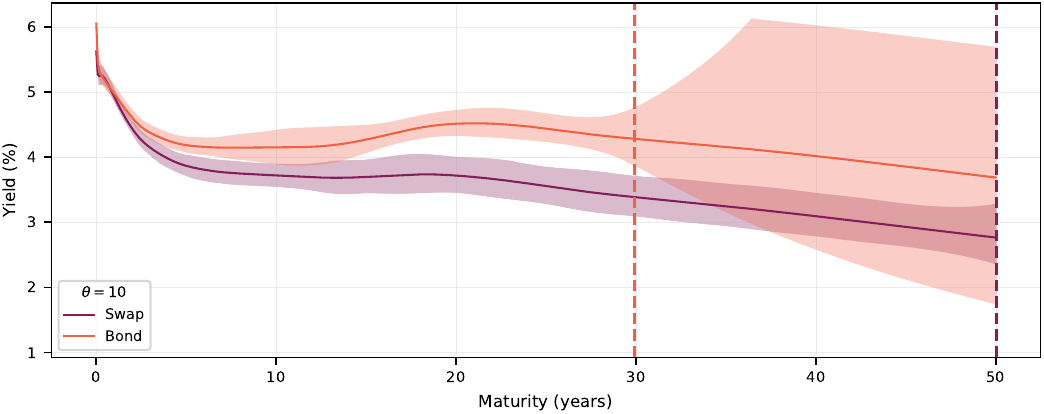}
  \end{subfigure}\hfill
  \begin{subfigure}[t]{0.49\linewidth}
    \centering
    \includegraphics[width=\linewidth]{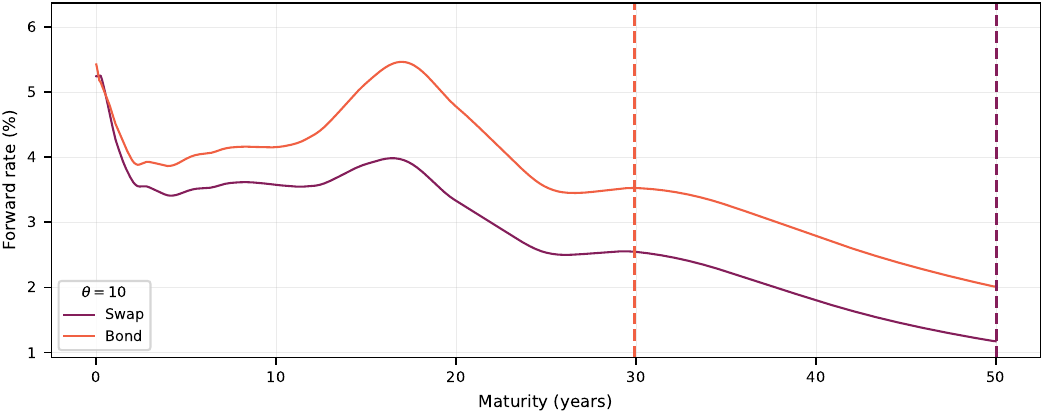}
  \end{subfigure}\vspace{0.3em}
    \begin{subfigure}[t]{0.49\linewidth}
    \centering
    \includegraphics[width=\linewidth]{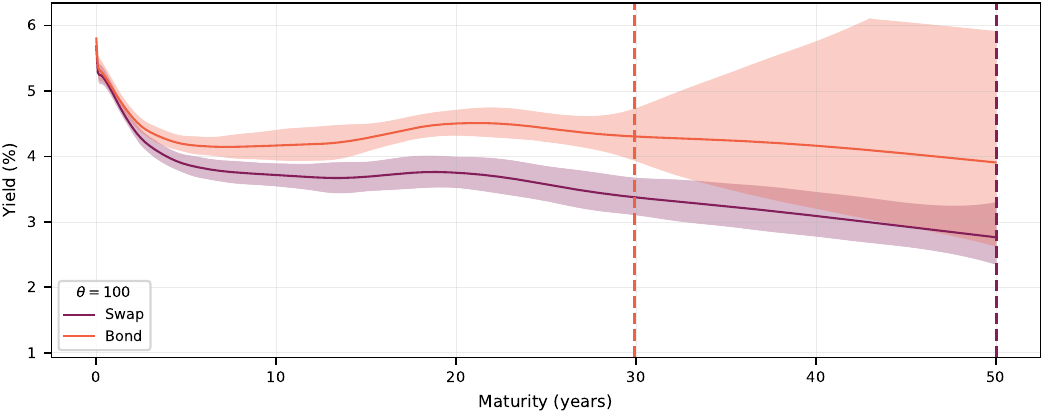}
  \end{subfigure}\hfill
  \begin{subfigure}[t]{0.49\linewidth}
    \centering
    \includegraphics[width=\linewidth]{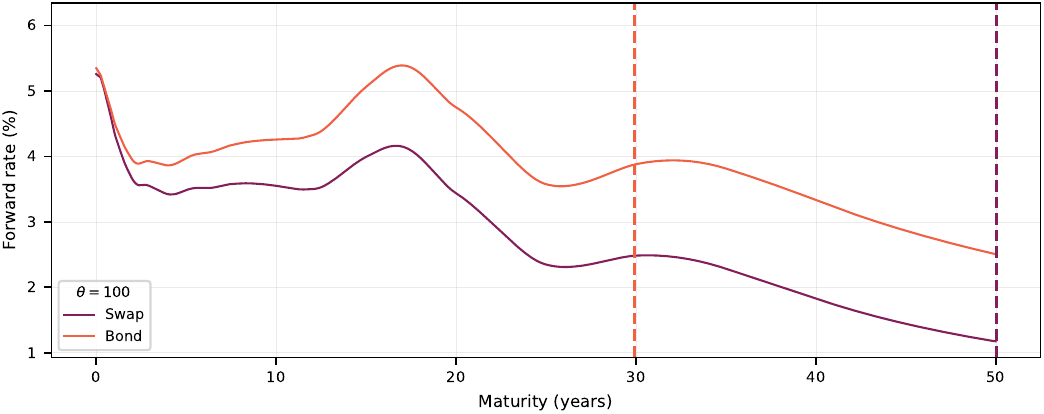}
  \end{subfigure}\vspace{0.3em}
    \begin{subfigure}[t]{0.49\linewidth}
    \centering
    \includegraphics[width=\linewidth]{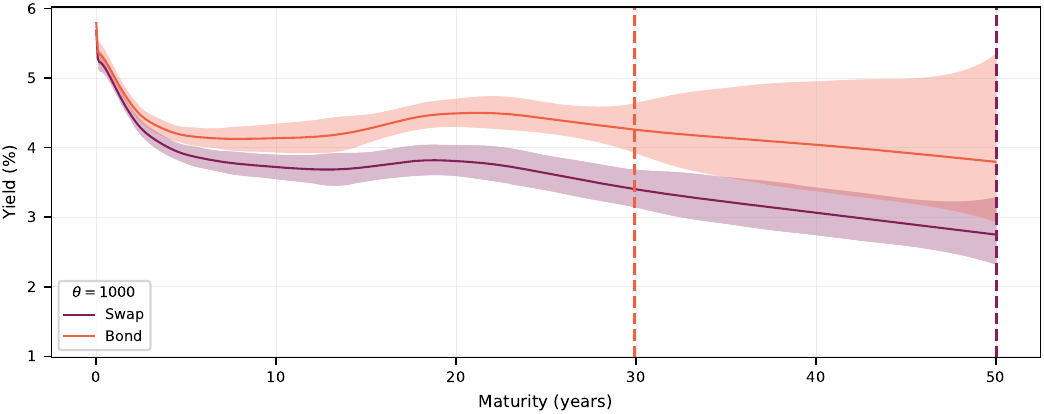}
  \end{subfigure}\hfill
  \begin{subfigure}[t]{0.49\linewidth}
    \centering
    \includegraphics[width=\linewidth]{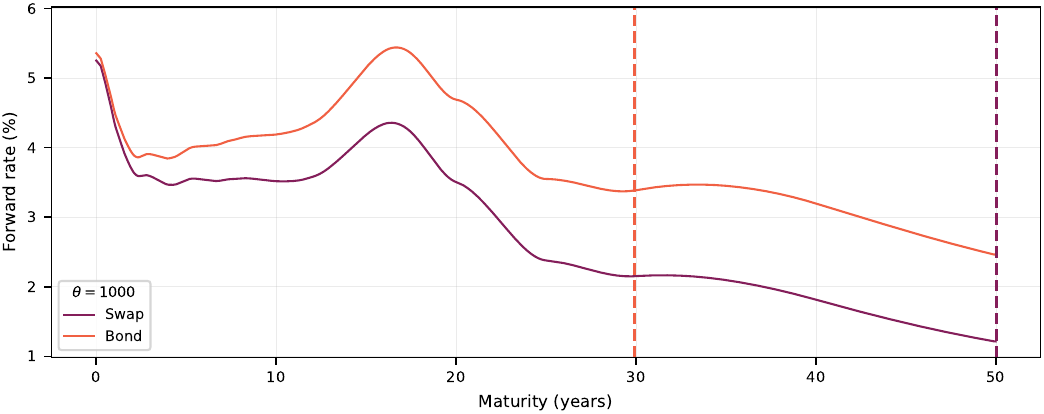}
  \end{subfigure}\vspace{0.3em}
        \bnotefig{This figure shows the resulting yield curves for various $\theta$ on the left and respective forward rate curves on the right on 2024-06-14.  In all panels, the vertical dashed lines indicate the longest available data point in the respective product class. The shaded areas show the $3\sigma$ confidence bands derived from the Gaussian process view and are capped at $\pm 2\%$. All values are in \%.} 
   \label{fig:example_day_2024}
\end{figure}

Additional example days are reported in Appendix~\ref{sec:additional_results}. They confirm the patterns observed in Figure~\ref{fig:example_day_2024}, with the magnitude of the effects varying across dates. Having established these illustrative insights, we now turn to a systematic evaluation of the benefits of transfer learning.

\subsection{Masking Experiment}
\label{sec:masking_experiment}

The remaining hyperparameter to select is $\theta$, which governs the strength of transfer learning. As is apparent from the additional regularization term in~\eqref{eq:multi_curve_motivated}, incorporating transfer learning generally leads to slightly higher in-sample pricing errors relative to the standalone case. The rationale for introducing transfer learning is therefore not improved in-sample fit, but enhanced estimation in regions where data are sparse or entirely unavailable.

To quantify this effect, we conduct a masking experiment. For a given masking horizon $H = 10$~years, all bonds with maturities exceeding $H$ are temporarily treated as unobserved, while swap data remain fully included.\footnote{Alternative choices, such as $H = 5$ and $H = 15$~years, yield qualitatively similar results and are available from the authors upon request.} Discount curves are then estimated using the unmasked bond data and the full swap sample, both in the standalone case ($\theta = 0$) and under transfer learning with $\theta \in \{1, 5, 10, 50, 100, 500, 1000\}$, where values of $\theta$ are scaled by a factor of $10^{-4}$. Model performance is evaluated using RMSE computed across all instruments. Figure~\ref{fig:experiment_design} illustrates the experimental design.

\begin{figure}[h!]
  \centering
       \tcapfig{Transfer learning masking experiment}
\begin{tikzpicture}[scale=0.9,>=stealth]

  \draw[->,thick] (0,0) -- (11.5,0) node[right]{Maturity};

  \draw[thick] (0,0) -- (0,-0.1); 
  \node[below] at (0,0) {0};
  
  \draw[thick,blue] (0,1.0) -- (4,1.0);
  \node[blue,above] at (2,1.0) {Bonds used};
  
  \draw[dashed,thick, blue] (4.1,1.0) -- (9.7,1.0);
  \node[above] at (7,1.0) {Masked bonds};
  
  \draw[ thick, red] (4,-0.1) -- (4,1.3);
  \node[red,below] at (4,0) {Masking horizon 10y};
  
  \draw[thick,green!60!black] (0,0.2) -- (11,0.2);
  \node[green!60!black,above] at (7,0.2) {Swaps used fully};
\end{tikzpicture}
        \bnotefig{The figure illustrates the design of the masking experiment to evaluate the benefits of transfer learning. } 
   \label{fig:experiment_design}
\end{figure}

We expect transfer learning to leave swap fitting errors largely unchanged, as swap data are fully observed throughout the experiment. For bonds, the fit in the unmasked region below~$H$ should be only marginally affected, while the extrapolated segment beyond~$H$ should improve substantially due to the additional information provided by swaps. This is precisely what we observe.

Results are reported both by maturity bucket, to highlight local effects, and in aggregate, to assess whether potential deterioration in well-identified regions is outweighed by gains in the extrapolation area. The masking experiment is conducted on a daily basis over the full sample period, allowing us to trace the effects of transfer learning over time. We first present a series of figures illustrating the improvements achieved through transfer learning, followed by a tabular summary of the results. In all figures, TL denotes transfer learning and SA refers to the standalone estimation. As a reference, we also report the SA unmasked benchmark, which corresponds to a standalone fit without bond masking.

Figure~\ref{fig:transfer_learning_line_mat_buckets} reports the time-averaged RMSE by maturity bucket for bonds (left panel) and swaps (right panel). For bonds, all curves remain tightly close for maturities up to 10~years. Only the largest values of~$\theta$ exhibit slight dispersion, indicating excessive transfer learning, although the differences remain within the low single–basis-point range. In this short- to medium-maturity segment, the SA masked specification typically delivers the lowest errors among the masked cases, aside from the SA unmasked benchmark. This pattern changes markedly in the masked region: the SA masked benchmark now performs worst, with RMSE exceeding 20bps, while the minimum error is attained at $\theta=100$, at approximately 13bps. Turning to the swap panel, modest distortions appear below 10~years for all transfer learning specifications, reflecting spillovers from bond information into the swap curve. The magnitude of these effects increases with~$\theta$ but remains small overall. Taken together, both panels indicate clear benefits of transfer learning for moderate values of~$\theta$, while excessively large values lead to adverse effects.

\begin{figure}[ht]
  \centering
       \tcapfig{Time-averaged fitting errors by maturity bucket and overall}
  \begin{subfigure}[t]{0.49\linewidth}
    \centering
    \includegraphics[width=\linewidth]{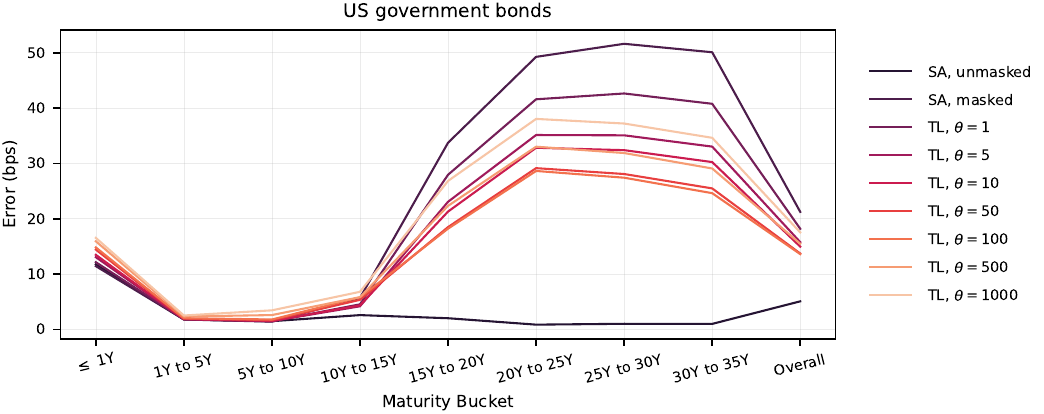}
  \end{subfigure}\hfill
  \begin{subfigure}[t]{0.49\linewidth}
    \centering
    \includegraphics[width=\linewidth]{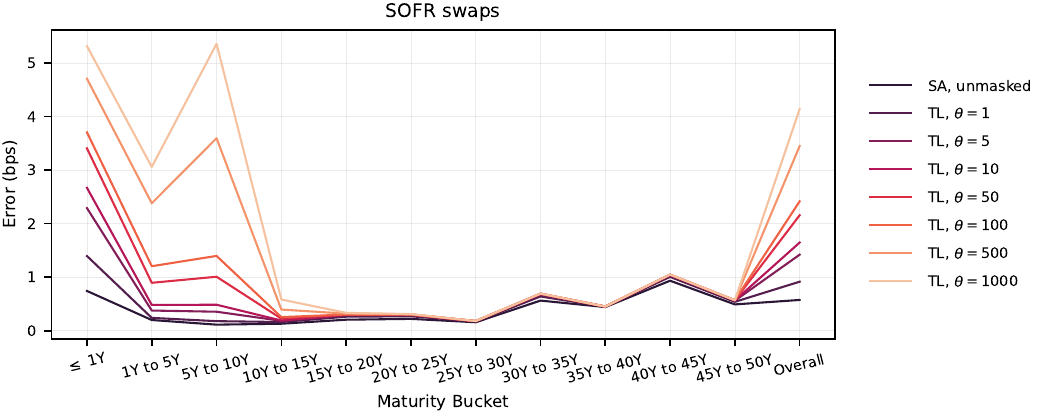}
  \end{subfigure}\vspace{0.3em}
    \bnotefig{This figure shows fitting  errors of US government bonds on the left and SOFR swaps on the right. The different colored lines correspond to different values of $\theta$. The masking horizon is $H=10$. The Overall bucket is the total aggregate. All values are in bps.} 
   \label{fig:transfer_learning_line_mat_buckets}
\end{figure}

Figure~\ref{fig:transfer_learning_box_mat_bucket} presents the corresponding logarithmic error distributions by bucket as well as in aggregate. This more granular perspective corroborates the findings discussed above. A small number of outliers, indicated by dots, are visible, but the overall patterns are stable across specifications. For bonds, the boxplots reveal a slight upward shift in the error distribution for maturities below 10~years. Beyond 10~years, the distributions display a smooth, smile-shaped pattern across values of~$\theta$, indicating an optimal range around $\theta \approx 100$. In contrast, the swap results exhibit the expected stability, with only minor distortions below 10~years, and of negligible economic magnitude.

\begin{figure}[ht]
  \centering
       \tcapfig{Distribution of log fitting errors by maturity bucket and overall}
    \begin{subfigure}[t]{0.49\linewidth}
    \centering
    \includegraphics[width=\linewidth]{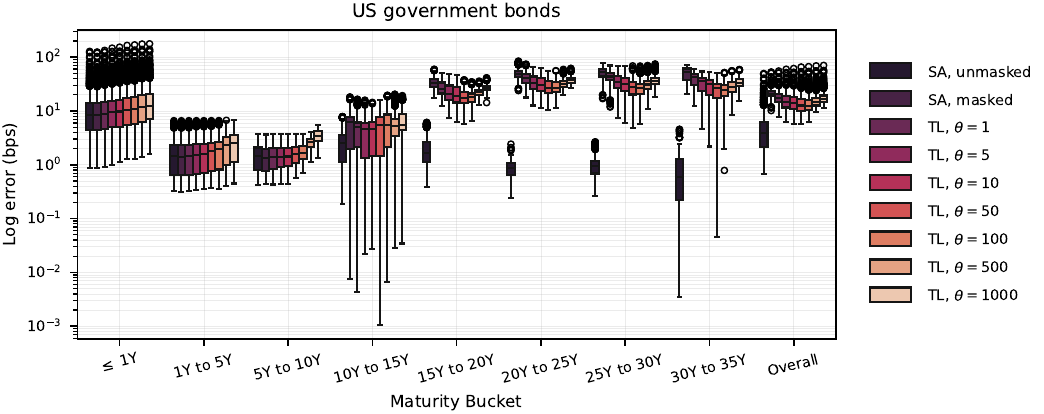}
  \end{subfigure}\hfill
  \begin{subfigure}[t]{0.49\linewidth}
    \centering
    \includegraphics[width=\linewidth]{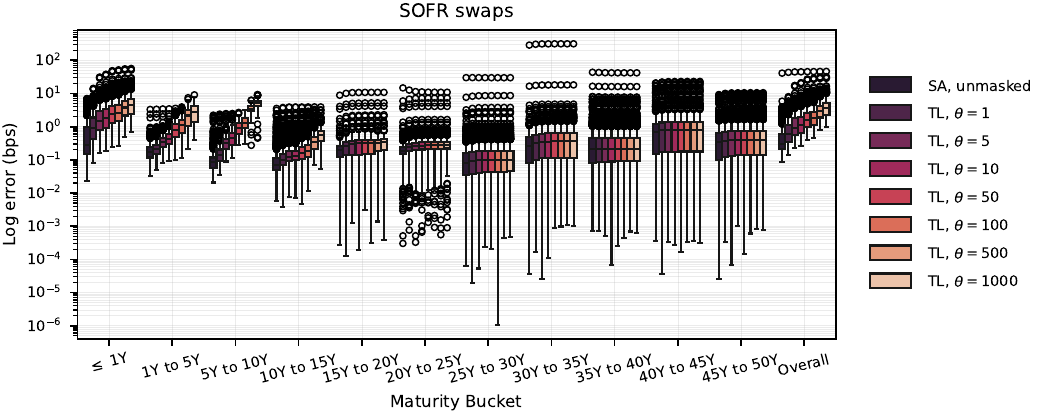}
  \end{subfigure}\vspace{0.3em}
        \bnotefig{This figure shows the distribution of log fitting errors of US government bonds on the left and  SOFR swaps on the right. The different colored whisker plots correspond to different values of $\theta$. The masking horizon is $H=10$. The Overall bucket is the total aggregate. All values are in log bps.} 
   \label{fig:transfer_learning_box_mat_bucket}
\end{figure}

The results thus far are encouraging. Transfer learning delivers clear improvements in the masked region while leaving other maturity segments largely unaffected for moderate values of~$\theta$. The maturity-bucket error distributions confirm these findings at a more granular level. To assess whether these improvements persist over time, Figure~\ref{fig:transfer_learning_time_series} plots the aggregated RMSE as a time series. The SA unmasked benchmark provides a natural lower bound, while all transfer learning specifications consistently outperform the SA masked case throughout the sample period. These patterns indicate that the gains from transfer learning are both robust and temporally stable. For swaps, all series remain closely aligned over time, whereas occasional spikes observed for both product classes are attributable to data noise rather than methodological shortcomings, a point we have verified systematically. 

  \begin{figure}[ht]
  \centering
       \tcapfig{Time-series of overall fitting errors}
    \begin{subfigure}[t]{0.49\linewidth}
    \centering
    \includegraphics[width=\linewidth]{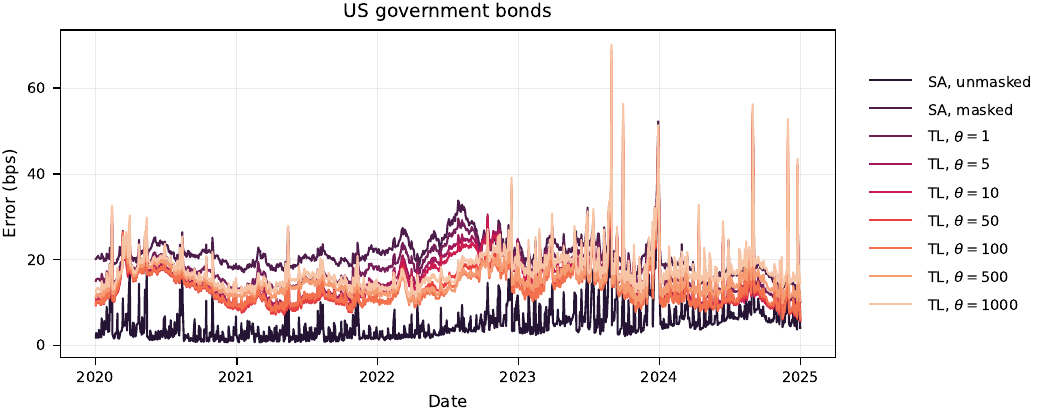}
  \end{subfigure}\hfill
  \begin{subfigure}[t]{0.49\linewidth}
    \centering
    \includegraphics[width=\linewidth]{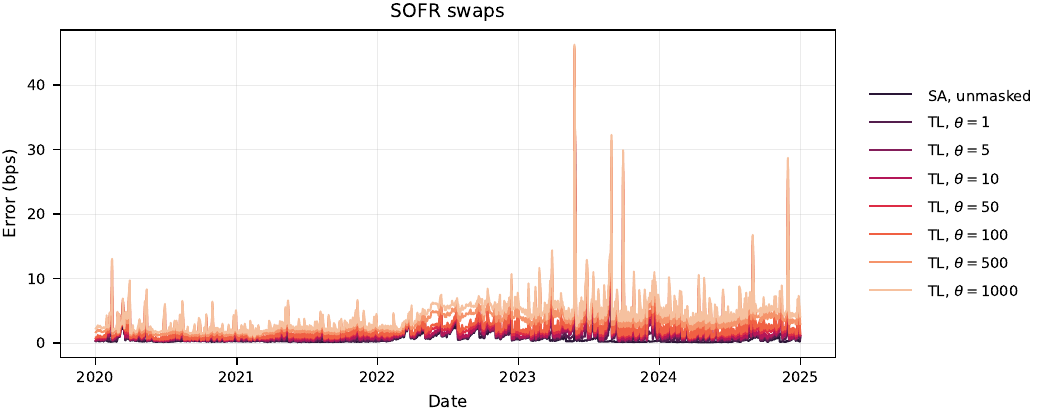}
    \end{subfigure}\vspace{0.3em}
       \bnotefig{This figure shows the time series of overall fitting errors of US government bonds on the left and SOFR swaps on the right. The different colored lines correspond to different values of $\theta$. The masking horizon is $H=10$. All values are in bps.} 
   \label{fig:transfer_learning_time_series}
\end{figure}

Table~\ref{tab:tl_effects} reports average and median fitting errors, measured in basis points. For reference, the upper panel presents results for the unmasked setting, in which no instruments are removed. As expected, fitting errors in this case are substantially smaller. The key observation is that, in the unmasked setup, transfer learning does not lead to any economically significant deterioration in fit quality. For bonds, the average error increases only marginally when moving from the SA unmasked specification to transfer learning with $\theta = 100$ (the RMSE-minimizing value), and the same pattern holds for swaps. An analogous conclusion emerges for median errors. In contrast, the masked results highlight the clear benefits of transfer learning. The minimum average error is attained at $\theta = 100$ (highlighted in green), reducing bond fitting errors by approximately 8bps relative to the standalone masked case, corresponding to a reduction of about 36\%. In light of the Bid--Ask spreads documented in Figure~\ref{fig:bid_ask}, this improvement is economically significant. At the same time, distortions for swaps remain modest: average errors increase only slightly, from 0.57bps in the SA masked case to 2.42bps under transfer learning with $\theta = 100$. The same qualitative conclusions hold when considering median errors.

\begin{table}[htbp] 
\centering 
\begin{tabular}{lcc@{\hspace{1em}}cc} 
\toprule & \multicolumn{2}{c}{Average} & \multicolumn{2}{c}{Median} \\ \cmidrule(lr){2-3} \cmidrule(lr){4-5} & \textbf{Bonds} & \textbf{Swaps} & \textbf{Bonds} & \textbf{Swaps} \\ 
\midrule \textbf{Unmasked} & & & & \\ 
SA & 5.08 & 0.57 & 3.91 & 0.29\\ 
TL, $\theta = 1$ & 5.39 & 0.89 & 4.12 & 0.57 \\ 
TL, $\theta = 5$ &  5.82 & 1.41 & 4.41 & 1.00\\ 
TL, $\theta = 10$ & 6.03 & 1.66 & 4.57 & 1.21\\ 
TL, $\theta = 50$ & 6.51 & 2.22 & 4.99 & 1.78 \\ 
TL, $\theta = 100$ & 6.72 & 2.49 & 5.19 & 2.04 \\ 
TL, $\theta = 500$ & 7.33 & 3.58 & 5.78 & 3.19\\ 
TL, $\theta = 1000$ & 7.73 & 4.40 & 6.18 & 4.01\\ 
\midrule \textbf{Masked} & & & & \\ 
SA & 21.24 & 0.57 & 20.84 & 0.29  \\ 
TL, $\theta = 1$ & 18.18 & 0.92 & 17.56 & 0.59   \\ 
TL, $\theta = 5$ & 15.78 & 1.42 & 14.76 & 0.94\\ 
TL, $\theta = 10$ & 14.96 & 1.65 & 13.98 & 1.11 \\ 
TL, $\theta = 50$ & 13.72 & 2.16 & 12.63 & 1.58\\ 
TL, $\theta = 100$ & \textbf{\textcolor{green!60!black}{13.61}} & 2.42 & \textbf{\textcolor{green!60!black}{12.25}} & 1.88 \\ 
TL, $\theta = 500$ & 15.49 & 3.45 & 14.38 & 3.02 \\ 
TL, $\theta = 1000$ & 17.53 & 4.14 & 16.48 & 3.76 \\ 
\bottomrule 
\end{tabular} 
\caption{The average and median fitting errors for US government bonds and SOFR swaps. Green highlights indicate the minimal masked bond errors. All values are in bps.} \label{tab:tl_effects} \end{table}

Thus far, we have assumed that bond and swap curves are estimated jointly via transfer learning. From a practical standpoint, however, this is not always necessary. Standalone KR curves already perform well across markets and relative to benchmark methods \cite{Filipovic2022,CamenzindFilipovic2024}. When data are sufficiently dense, the standalone approach remains preferable. When data become sparse or unavailable, transfer learning provides a simple and effective extension that materially improves extrapolation, in particular for extending the bond curve up to the longest available swap maturity of 50~years. In this sense, the transfer-learned swap curve is best viewed as a by-product of the procedure rather than a replacement for its standalone estimation.

In summary, transfer learning preserves fit quality in well-populated regions while delivering economically meaningful improvements in data-scarce segments. These gains are robust across maturity buckets and stable over time.

\section{Conclusion}

We introduce a transfer learning framework for jointly estimating discount curves across fixed-income product classes. Building on the discounted cash flow principle, our approach extends kernel ridge regression to a vector-valued setting, resulting in a convex optimization problem with a closed-form solution in a vector-valued RKHS. A key feature is the use of separable operator-valued kernels, which enable regularization of curve spreads in an economically meaningful way.

We derive a norm decomposition for separable kernels, generalizing prior results and yielding a principled spread regularization term. The framework admits a Gaussian process interpretation, allowing for estimation uncertainty quantification in the vector-valued setting.

We show how standard fixed-income instruments, including coupon bonds, interest-rate swaps, and cross-currency swaps, can be embedded within this framework. An extensive masking experiment demonstrates that transfer learning US government bonds with SOFR swaps improves extrapolation while leaving well-identified regions unaffected. The resulting effects are economically significant and consistent across maturity buckets and over time. A comprehensive empirical assessment for additional currencies is left for future work.


\bibliographystyle{alpha}
\bibliography{bibliography}

\newcommand{\etalchar}[1]{$^{#1}$}
\begin{thebibliography}{KDP{\etalchar{+}}16}

\bibitem[ARL12]{alv_ros_law_11}
Mauricio~A. Alvarez, Lorenzo Rosasco, and Neil~D. Lawrence.
\newblock Kernels for vector-valued functions: A review.
\newblock {\em Foundations and Trends in Machine Learning}, 4(3):195--266,
  2012.

\bibitem[BHJ{\etalchar{+}}19]{ECBXCCY}
Thomas Brophy, Niko Herrala, Raquel Jurado, Irene Katsalirou, Léa Le~Quéau,
  Christian Lizarazo, and Seamus O’Donnell.
\newblock {Role of cross currency swap markets in funding and investment
  decisions}.
\newblock Occasional Paper Series 228, European Central Bank, August 2019.

\bibitem[Bjo09]{bjo_09}
Tomas Bjork.
\newblock {\em {Arbitrage Theory in Continuous Time}}.
\newblock Number 9780199574742 in OUP Catalogue. Oxford University Press, 2009.

\bibitem[{B}o{E}]{BOESoniaOverhaul}
{B}o{E}.
\newblock {SONIA} key features and policies.
\newblock
  \url{https://www.bankofengland.co.uk/markets/sonia-benchmark/sonia-key-features-and-policies}
  (accessed: 08.05.2025).

\bibitem[BRBV12]{baldassarre:multi_output_learning}
Luca Baldassarre, Lorenzo Rosasco, Annalisa Barla, and Alessandro Verri.
\newblock Multi-output learning via spectral filtering.
\newblock {\em Machine Learning}, 87, 2012.

\bibitem[Car97]{Caruana1997}
Rich Caruana.
\newblock Multitask learning.
\newblock {\em Machine Learning}, 28(1):41–75, 1997.

\bibitem[CF24]{CamenzindFilipovic2024}
Nicolas Camenzind and Damir Filipović.
\newblock Stripping the swiss discount curve using kernel ridge regression.
\newblock {\em European Actuarial Journal}, 14(2):371–410, June 2024.

\bibitem[{Chr}17]{clarusft2017}
{Chris Barnes, Clarus FT}.
\newblock Mechanics of cross currency swaps, 2017.
\newblock \url{https://www.clarusft.com/mechanics-of-cross-currency-swaps/}
  (accessed: 08.05.2025).

\bibitem[CWG19]{Chen2019}
Zexun Chen, Bo~Wang, and Alexander~N. Gorban.
\newblock Multivariate gaussian and student-t process regression for
  multi-output prediction.
\newblock {\em Neural Computing and Applications}, 32(8):3005--3028, dec 2019.

\bibitem[ECB]{ECBEstr}
ECB.
\newblock Euro short-term rate {(€STR)}.
\newblock
  \url{https://www.ecb.europa.eu/stats/financial_markets_and_interest_rates/euro_short-term_rate/html/index.de.html}
  (accessed: 08.05.2025).

\bibitem[FB87]{FAMABLISS1987}
Eugene~F. Fama and Robert~R. Bliss.
\newblock The information in long-maturity forward rates.
\newblock {\em The American Economic Review}, 77(4):680--692, 1987.

\bibitem[Fed]{FedSofr}
Fed.
\newblock Secured {O}vernight {F}inancing {R}ate {(SOFR)}.
\newblock \url{https://www.newyorkfed.org/markets/reference-rates/sofr}
  (accessed: 08.05.2025).

\bibitem[FLT17]{Filipovic2017-wh}
Damir Filipovi{\'c}, Martin Larsson, and Anders~B Trolle.
\newblock Linear-rational term structure models.
\newblock {\em J. Finance}, 72(2):655--704, April 2017.

\bibitem[FPY24]{Filipovic2022}
Damir Filipovic, Markus Pelger, and Ye~Ye.
\newblock Stripping the discount curve --- a robust machine learning approach.
\newblock {\em Management Science}, 2024.
\newblock Accepted for publication.

\bibitem[GSW07a]{gur_sac_wri_07}
Refet~S. G{\"{u}}rkaynak, Brian Sack, and Jonathan~H. Wright.
\newblock {The U.S. Treasury yield curve: 1961 to the present}.
\newblock {\em Journal of Monetary Economics}, 54(8):2291--2304, 2007.

\bibitem[GSW07b]{GURKAYNAK20072291}
Refet~S. Gürkaynak, Brian Sack, and Jonathan~H. Wright.
\newblock The {U}.{S}. treasury yield curve: 1961 to the present.
\newblock {\em Journal of Monetary Economics}, 54(8):2291--2304, 2007.

\bibitem[HJ12]{horn_johnson_2012}
Roger~A. Horn and Charles~R. Johnson.
\newblock {\em Matrix Analysis}.
\newblock Cambridge University Press, 2 edition, 2012.

\bibitem[{Int}20]{icma2020rfrguide}
{International Capital Market Association}.
\newblock A quick guide to the transition to risk-free rates in the
  international bond market, 2020.
\newblock
  \href{www.icmagroup.org/assets/documents/Regulatory/Benchmark-reform/A-quick-guide-to-the-transition-to-risk-free-rates-in-the-international-bond-market-February-2020-27022020.pdf}{pdf}
  (accessed: 08.05.2025).

\bibitem[JT95]{Jarrow1995-ss}
Robert~A Jarrow and Stuart~M Turnbull.
\newblock Pricing derivatives on financial securities subject to credit risk.
\newblock {\em J. Finance}, 50(1):53, March 1995.

\bibitem[Kat95]{kat_95}
Tosio Kato.
\newblock {\em Perturbation theory for linear operators}.
\newblock Classics in Mathematics. Springer-Verlag, Berlin, 1995.
\newblock Reprint of the 1980 edition.

\bibitem[KDP{\etalchar{+}}16]{kadri_duflos_2016}
Hachem Kadri, Emmanuel Duflos, Philippe Preux, St{{\'e}}phane Canu, Alain
  Rakotomamonjy, and Julien Audiffren.
\newblock Operator-valued kernels for learning from functional response data.
\newblock {\em Journal of Machine Learning Research}, 17(20):1--54, 2016.

\bibitem[LW21]{LIU20211395}
Yan Liu and Jing~Cynthia Wu.
\newblock Reconstructing the yield curve.
\newblock {\em Journal of Financial Economics}, 142(3):1395--1425, 2021.

\bibitem[MM]{SNBSaron}
Andréa~M. Maechler and Thomas Moser.
\newblock Life after {L}ibor: {A} new era of reference interest rates.
\newblock
  \url{https://www.snb.ch/en/publications/communication/speeches/2022/ref_20220331_amrtmo}
  (accessed: 08.05.2025).

\bibitem[MP04]{pontil:kernels_multi_task_learning}
Charles~A. Micchelli and Massimiliano Pontil.
\newblock Kernels for multi-task learning.
\newblock {\em NIPS}, 2004.

\bibitem[MP05]{pontil:on_learning_vvf}
Charles~A. Micchelli and Massimiliano Pontil.
\newblock On learning vector-valued functions.
\newblock {\em Neural Computation}, 17, 2005.

\bibitem[NS87]{Nelson1987}
Charles~R. Nelson and Andrew~F. Siegel.
\newblock Parsimonious modeling of yield curves.
\newblock {\em The Journal of Business}, 60(4):473, January 1987.

\bibitem[PR16]{pau_rag_16}
Vern~I. Paulsen and Mrinal Raghupathi.
\newblock {\em An introduction to the theory of reproducing kernel {H}ilbert
  spaces}, volume 152 of {\em Cambridge Studies in Advanced Mathematics}.
\newblock Cambridge University Press, Cambridge, 2016.

\bibitem[PY10]{Pan2010-qg}
Sinno~Jialin Pan and Qiang Yang.
\newblock A survey on transfer learning.
\newblock {\em IEEE Trans. Knowl. Data Eng.}, 22(10):1345--1359, October 2010.

\bibitem[Ran23]{Ranaldo2022XCCY}
Angelo Ranaldo.
\newblock {Foreign exchange swaps and cross-currency swaps}.
\newblock In Refet~S. Gürkaynak and Jonathan~H. Wright, editors, {\em
  {Research Handbook of Financial Markets}}, Chapters, chapter~20, pages
  451--469. Edward Elgar Publishing, 2023.

\bibitem[RW05]{Rasmussen2005-rn}
Carl~Edward Rasmussen and Christopher K~I Williams.
\newblock {\em Gaussian processes for machine learning}.
\newblock Adaptive Computation and Machine Learning series. MIT Press, London,
  England, November 2005.

\bibitem[She08]{sheldon:graph_multi_task_learning}
Daniel Sheldon.
\newblock Graphical multi-task learning.
\newblock Technical report, Cornell University, 2008.

\bibitem[SIX]{SixSaron}
SIX.
\newblock Swiss {R}eference {R}ates {(SARON)}.
\newblock
  \url{https://www.six-group.com/en/market-data/indices/switzerland/saron.html}
  (accessed: 08.05.2025).

\bibitem[SK03]{Smola2003}
Alexander~J. Smola and Risi Kondor.
\newblock {\em Kernels and Regularization on Graphs}, page 144–158.
\newblock Springer Berlin Heidelberg, 2003.

\bibitem[SS19]{Schrimpf2019}
Andreas Schrimpf and Vladyslav Sushko.
\newblock Beyond libor: a primer on the new benchmark rates.
\newblock {\em BIS Quarterly Review}, 2019.

\bibitem[Sve94]{SVENSSON1944}
Lars Svensson.
\newblock Estimating and interpreting forward interest rates: Sweden 1992 -
  1994.
\newblock NBER Working Papers 4871, National Bureau of Economic Research, Inc,
  1994.

\bibitem[SW01]{smi_wil_01}
Andrew Smith and Tim Wilson.
\newblock Fitting yield curves with long term constraints.
\newblock {\em Working paper}, 2001.

\bibitem[Tea21]{risk2021sofr}
Risk.net~Editorial Team.
\newblock Beyond libor: the impact of sofr on rates, bonds and loans, 2021.
\newblock
  \url{https://www.risk.net/insight/markets/7957455/beyond-libor-the-impact-of-sofr-on-rates-bonds-and-loans}
  (accessed: 08.05.2025).

\bibitem[{Wik}]{wiki_Kmfvo}
{Wikipedia contributors}.
\newblock Kernel methods for vector output.
\newblock \url{https://en.wikipedia.org/wiki/Kernel_methods_for_vector_output}
  (accessed: 08.05.2025).

\bibitem[WJ24]{wu_jar_24}
David Wu and Robert~A. Jarrow.
\newblock The {T}reasury--{SOFR} swap spread puzzle explained, 2024.
\newblock Available at SSRN: \url{https://ssrn.com/abstract=4904777}.

\bibitem[WKW16]{Weiss2016-kn}
Karl Weiss, Taghi~M Khoshgoftaar, and Dingding Wang.
\newblock A survey of transfer learning.
\newblock {\em J. Big Data}, 3(1), December 2016.

\end{thebibliography}

\begin{appendix}

\section{Vector-Valued Reproducing Kernel Hilbert Spaces}\label{secRKHS}

This appendix presents the theoretical background on vector-valued RKHS that underpins our transfer learning framework. For completeness, we begin by recalling the definition and main properties of vector-valued RKHS, following \cite[Chapter 6]{pau_rag_16}. Let $E$ be any set and $A\in\N$.

\begin{definition}
  An \emph{$\R^A$-valued RKHS} on $E$ is a Hilbert space $\Hcal$ consisting of functions $h=(h_1,\dots,h_A)^\top:E\to\R^A$ such that for every $x\in E$, the linear evaluation map $E_x:\Hcal\to\R^A$ given by $E_x(h)=h(x)$ is bounded.
\end{definition}

An $\R^A$-valued RKHS $\Hcal$ has a reproducing kernel function $K:E\times E\to \R^{A\times A}$ defined by $K(x,y)=E_x E_y^\ast$, where we identify a linear operator on $\R^A$ with its $A\times A$-matrix representation in the standard Euclidean basis of $\R^A$. $E_y^\ast$ denotes the adjoint operator. We immediately obtain that $K(\cdot,y)v=E_y^\ast v \in\Hcal$ and $\langle K(\cdot,y)v, h\rangle_{\Hcal}=v^\top h(y)$, for any $y\in E$, $v\in \R^A$, $h\in\Hcal$. Moreover, we see that $K$ is symmetric in the following sense,
\begin{equation}\label{Ktop}
  K(x,y)^\top = K(y,x) .
\end{equation}
Note, however, that the matrices $K(x,y)$ are not symmetric for $x\neq y$ in general.\footnote{Many papers in the literature assume that the matrices $K(x,y)$ are symmetric. But this is not the case in general, and, in fact, it excludes many examples.} Moreover, for any finite points $x_1,\dots,x_n\in E$ the operator $(K(x_i,x_j))$ on $(\R^A)^A$ is positive semi-definite in the sense that for all choices of vectors $v_1,\dots,v_n\in\R^A$ we have
\begin{equation}\label{PSD}
  \sum_{i,j=1}^n v_i^\top K(x_i,x_j) v_j \ge 0.
\end{equation}

Conversely, this leads to the following definition.\footnote{Note that \cite[Definition 6.11]{pau_rag_16} does not require \eqref{Ktop} because they work on complex Hilbert spaces, where the non-negativity, that is, \eqref{PSD} with $v_i$ replaced by its complex conjugate, already implies that $K(x,y)^\ast = K(y,x)$.}
\begin{definition}
  A function $K:E\times E\to\R^{A\times A}$ satisfying \eqref{Ktop} and \eqref{PSD} is called a \emph{$\R^{A\times A}$-valued kernel function}.
\end{definition}

It follows by inspection that a function $K:E\times E\to\R^{A\times A}$ is a $\R^{A\times A}$-valued kernel function if and only if there exists a scalar kernel function $k$ on $\{1,\dots,A\}\times E$ such that $K_{ab}(x,y) = k((a,x),(b,y))$.

\begin{example} The concept of matrix-valued kernels is surprisingly strong as it is somewhat difficult to generate examples easily.
  However, one possible way is to let $k_1, k_2,\dots, k_A$ be scalar kernels on $E$, then $K(x,y)=\diag(k_1(x,y),\dots,k_A(x,y))$ is a $\R^{A\times A}$-valued kernel.
  Indeed, property \eqref{Ktop} holds because $K(x,y)=K(y,x)$ is symmetric.
  Property \eqref{PSD} is valid since
  \begin{align*}
    \sum_{i,j=1}^n v_i^\top K(x_i,x_j) v_j = \sum_{i,j=1}^n \sum_{a=1}^A v_{i,a} k_a(x_i,x_j) v_{j,a} = \sum_{a=1}^A\bigg(\underbrace{\sum_{i,j=1}^n  v_{i,a} k_a(x_i,x_j) v_{j,a}}_{\ge 0}\bigg) \ge 0,
  \end{align*}
where the inner sums are non-negative due to the kernel property of each scalar kernel $k_a$.
\end{example}
Moore's vector-valued theorem \cite[Theorem 6.12]{pau_rag_16} states that for every $\R^{A\times A}$-valued kernel function $K$ there exists a unique $\R^A$-valued RKHS $\Hcal$ such that $K$ is its reproducing kernel function. Moreover, functions of the form
\begin{equation}\label{reprH}
  h(x) = \sum_{j=1}^n K(x,y_j) v_j ,\quad v_j\in\R^{A}, \quad y_1,\dots, y_n\in E,\quad n\in\N,
\end{equation}
are dense in $\Hcal$, see \cite[Proposition 6.7]{pau_rag_16}.\footnote{Note that \eqref{reprH} differs from the corresponding formulas in \cite[page 209]{alv_ros_law_11} and the wikipedia page \cite{wiki_Kmfvo}. The latter formulas are correct only if $K(x,y)$ is a symmetric matrix, which in view of \eqref{Ktop} is not true in general.}
A special class of $\mathbb{R}^{A\times A}$-valued kernels are separable kernels. In fact, as it turns out they are tractable and easy to interpret. 
\begin{definition}
\label{def:separable_kernel}
  A $\R^{A\times A}$-valued kernel $K$ on $E$ is separable if it can be written as $K(x,y)= B k(x,y)$ for some $A\times A$-matrix $B$ and a scalar kernel $k$ on $E$. In view of \eqref{Ktop}, the matrix $B$ is necessarily symmetric positive semi-definite.
\end{definition}

\begin{remark}
    Separable kernels are one of the  simplest matrix-valued kernel. If we regard kernels as similarity measures, $B$ encodes similarity across components while $k$ encodes similarity across the space $E$.
\end{remark}

The following theorem provides an important representation result, which is at the heart of transfer learning in this paper.

\begin{theorem}\label{thmRHKSsepNEW}
Let $\Hcal$ be the vector-valued RKHS corresponding to the separable kernel $K(x,y)= B k(x,y)$. Let $\Hcal_k$ denote the RKHS corresponding to the scalar kernel $k$. Let $Q$ be any generalized inverse $A\times A$-matrix of $B$ such that $BQB=B$. Then the following hold.
\begin{enumerate}
    \item\label{thmRHKSsepNEWX1} $\Hcal$ is isomorphic to the direct sum $\bigoplus_{a=1}^{\tilde A}\Hcal_k$, where $\tilde A=\rank B$.
\item\label{thmRHKSsepNEWX2} $\Hcal\subseteq (\Hcal_k)^A=\Hcal_k\times\cdots\times\Hcal_k$ as sets, with equality if and only if $B$ is non-singular.
    \item\label{thmRHKSsepNEW3} For any $h=(h_1,\dots,h_A)^\top\in \Hcal$, the $\Hcal$-norm can be expressed as
  \begin{equation}
       \| h\|^2_\Hcal  = \sum_{a,b=1}^A Q_{ab}\langle h_a,h_b\rangle_{\Hcal_k} \label{HKQeqfullnew}.
  \end{equation}
\item\label{thmRHKSsepNEW4} If $Q$ is symmetric then \eqref{HKQeqfullnew} can also be written as
\begin{equation}
    \| h\|^2_\Hcal  = \sum_{a=1}^A\gamma_a\|h_a\|_{\Hcal_k}^2- \sum_{a=1}^A\sum_{b>a} Q_{ab}\|h_a-h_b\|_{\Hcal_k}^2,\label{HKQeqnew}
\end{equation}
where $\gamma_a=  \sum_{b=1}^A Q_{ab}$ denote the row sums.
\end{enumerate}
\end{theorem}

\begin{proof}
We define the linear subspace $\Dcal$ of $\Hcal$ that consists of all functions of the form 
\begin{equation}\label{hsimplen}
    h(\cdot)=\sum_{j=1}^n Bv_j k(\cdot,y_j), \quad v_j\in\R^A,\quad y_1,\dots,y_n\in E,\quad n\in\N.
\end{equation}
From \eqref{reprH} we know that $\Dcal$ is dense in $\Hcal$. Similarly, we define the dense subspace $\Dcal_k$ of $\Hcal_k$ of all functions of the form $g(\cdot)=\sum_{j=1}^n c_j k(\cdot,y_j)$, for $c_j\in\R$. Consequently, the direct sum $\bigoplus_{a=1}^{\tilde A}\Dcal_k$ is a dense subspace of $\bigoplus_{a=1}^{\tilde A}\Hcal_k$.

We prove the theorem in two steps. First, we prove all statements for $\Hcal$ and $\Hcal_k$ replaced by $\Dcal$ and $\Dcal_k$. Second, we argue by the continuous extension principle that all results carry over to $\Hcal$ and $\Hcal_k$.

We let $B=USU^\top$ denote the reduced spectral decomposition where $U$ is an orthogonal $A\times \tilde A$-matrix such that $U^\top U = I_{\tilde A}$, and $S=\diag(s_1,\dots,s_{\tilde A})$ contains the positive eigenvalues $s_1\ge \cdots \ge s_{\tilde A}>0$ of $B$. 

We define the linear operator $\Ucal:\Dcal\to \bigoplus_{a=1}^{\tilde A}\Dcal_k $, by $\Ucal h(\cdot)  =U^\top \sum_{j=1}^n   Bv_j k(\cdot,y_j) = \sum_{j=1}^n  S U^\top v_j k(\cdot,y_j)$. The operator $\Ucal$ is injective, because $\Ucal h(\cdot)=0$ implies that $SU^\top v_j=0$ and thus $v_j=0$ for all $j=1,\dots,n$, hence $h=0$. Here we assume that $n$ is minimal in the sense that $k(\cdot,y_1),\dots,k(\cdot,y_n)$ are linearly independent in $\Hcal_k$, without loss of generality. We claim that $\Ucal$ is also surjective, $\Ucal(\Dcal)=\bigoplus_{a=1}^{\tilde A}\Dcal_k$. Indeed, any $g\in \bigoplus_{a=1}^{\tilde A}\Dcal_k$ can be written as $g(\cdot)=\sum_{j=1}^n w_j k(\cdot,y_j) $, for some $w_j\in\R^{\tilde A}$, $y_1,\dots,y_n\in E$, $n\in \N$. Define the linear operator $\Vcal: \bigoplus_{a=1}^{\tilde A}\Dcal_k \to \Dcal$ by $\Vcal g(\cdot) = U \sum_{j=1}^n  w_j k(\cdot,y_j) $. As the $\tilde A\times A$-matrix $SU^\top $ has full rank $\tilde A$, there exist $v_j\in\R^A$ such that $w_j=SU^\top v_j$. Then $h\in\Dcal$ given by $h(\cdot)=\Vcal g(\cdot)= \sum_{j=1}^n  B v_j k(\cdot,y_j)$ is a pre-image of $g$, because $\Ucal h(\cdot)= U^\top U \sum_{j=1}^n  w_j k(\cdot,y_j) = g(\cdot)$. We conclude that $\Ucal:\Dcal\to \bigoplus_{a=1}^{\tilde A}\Dcal_k$ is a linear bijection with inverse given by $\Ucal^{-1}=\Vcal$, which proves \ref{thmRHKSsepNEWX1}.

We also obtain that the components $h_a$ of any $h\in\Dcal$ are linear combinations of functions $g_b \in\Dcal_k$ and thus elements in $\Dcal_k$ themselves. As $\tilde A=A$ if and only if $B$ is non-singular, this proves \ref{thmRHKSsepNEWX2}.

Next we claim that \eqref{HKQeqfullnew} holds for $h\in\Dcal$. Indeed, on one hand we have
\[\| h\|^2_\Hcal = \sum_{i,j=1}^n \langle  B v_i k(\cdot,y_i) ,  Bv_j k(\cdot,y_j)\rangle_\Hcal =\sum_{i,j=1}^n  v_i^\top B v_j k(y_i,y_j) \]
by the basic reproducing kernel property of $K(\cdot,y_i) = B k(\cdot,y_i)$. On the other hand, the right hand side of \eqref{HKQeqfullnew} equals
\[ \sum_{a,b=1}^A Q_{ab}\langle h_a,h_b\rangle_{\Hcal_k}= \sum_{i,j=1}^n \sum_{a,b=1}^A Q_{ab} (Bv_i)_a (Bv_j)_b k(y_i,y_j)= \sum_{i,j=1}^n v_i^\top B  Q B v_j k(y_i,y_j),\]
which equals the former and thus proves \ref{thmRHKSsepNEW3}.

As for \eqref{HKQeqnew}, straightforward rearrangement of sums shows that the right hand side of \eqref{HKQeqnew} equals
  \begin{align*}
    RHS&=\sum_{a=1}^A  Q_{aa}\|h_a\|_{\Hcal_k}^2+\sum_{a=1}^A\sum_{b\neq a} Q_{ab}\left(\|h_a\|_{\Hcal_k}^2-\frac{1}{2}\|h_a-h_b\|_{\Hcal_k}^2\right) \\
       &= \sum_{a=1}^A  Q_{aa}\|h_a\|_{\Hcal_k}^2+\sum_{a=1}^A\sum_{b\neq a} Q_{ab}\langle h_a,b_b\rangle_{\Hcal_k}=\sum_{a,b=1}^A  Q_{ab}\langle h_a,h_b\rangle_{\Hcal_k}.
  \end{align*}
In view of \eqref{HKQeqfullnew}, this proves \ref{thmRHKSsepNEW4}.

We now extend the validity of the above proved properties to $\Hcal$ and $\Hcal_k$. Thereto, when writing $h=\Ucal^{-1} g$ for $g=\Ucal h\in \bigoplus_{a=1}^{\tilde A}\Dcal_k$, we observe that the right hand side of \eqref{HKQeqfullnew} becomes
\begin{equation}\label{hsumgH}
    \| h\|^2_\Hcal  = \sum_{a=1}^{\tilde A} s_a^{-1} \|g_a\|_{\Hcal_k}^2.
\end{equation}
Indeed, $BQB=B$ implies $U^\top Q U=S^{-1}$, and thus $v^\top Q v = w^\top S^{-1} w$ for any $v=Uw$, which shows \eqref{hsumgH}.\footnote{In more detail: we have $h_a=\sum_{i=1}^{\tilde A} U_{ai}g_i$, and hence $\sum_{a,b=1}^A Q_{ab}\langle h_a,h_b\rangle_{\Hcal_k} = \sum_{i,j=1}^{\tilde A}\sum_{a,b=1}^A U_{ai}Q_{ab}U_{bj}\langle g_i,g_j\rangle_{\Hcal_k}$, which equals the right hand side of \eqref{hsumgH}.} We obtain the bounds $s_1^{-1} \| g\|_{\bigoplus_{a=1}^{\tilde A}\Hcal_k}^2 \le \| h\|_\Hcal^2\le s_{\tilde A}^{-1} \| g\|_{\bigoplus_{a=1}^{\tilde A}\Hcal_k}^2$

We infer that $\Ucal:\Dcal\subset\Hcal \to \bigoplus_{a=1}^{\tilde A}\Hcal_k$ is bounded with operator norm $\|\Ucal\|= s_1$. In the same vein, $\Ucal^{-1}:\bigoplus_{a=1}^{\tilde A}\Dcal_k\subset \bigoplus_{a=1}^{\tilde A}\Hcal_k\to \Hcal$ is bounded with operator norm $\|\Ucal^{-1}\|= s_{\tilde A}^{-1}$. By the extension principle for bounded densely defined operators on Banach spaces, \cite[Section III.2.2]{kat_95}, $\Ucal$ uniquely extends to an invertible bounded operator $\Ucal:\Hcal \to \bigoplus_{a=1}^{\tilde A}\Hcal_k$ with inverse given by the respective extension of $\Ucal^{-1}$. As norm convergence in $\Hcal$ and $\bigoplus_{a=1}^{\tilde A}\Hcal_k $ implies point-wise convergence, we have $\Ucal h(\cdot) = U^\top h(\cdot)$ and $\Ucal^{-1}g(\cdot)=Ug(\cdot)$, for all $h\in\Hcal$ and $g\in \bigoplus_{a=1}^{\tilde A}\Hcal_k$. The validity of \ref{thmRHKSsepNEWX1}, \ref{thmRHKSsepNEWX2}, \ref{thmRHKSsepNEW3}, \ref{thmRHKSsepNEW4} for $\Hcal$ and $\Hcal_k$ now follows by continuity arguments.
\end{proof}

\begin{remark}
Equation \eqref{HKQeqfullnew} is also proved in \cite[Proposition 1]{baldassarre:multi_output_learning}, however, only for simple functions of the form \eqref{hsimplen}, which corresponds to the first step in our proof of Theorem \ref{thmRHKSsepNEW}. 
\end{remark}

The following two auxiliary lemmas are of independent interest and potentially useful for the specification of a matrix-valued kernel. They provide general elementary decomposition results, which are known as kernel normalization in the scalar case.
\begin{lemma}\label{lemnormdec}
  Any $\R^{A\times A}$-valued kernel $K$ can be decomposed in the following way
  \begin{equation}\label{eqnormdec}
    K(x,y) = S(x) R(x,y) S(y)
  \end{equation}
  where $R$ is a normalized $\mathbb{R}^{A\times A}$-valued kernel such that $R_{aa}(x,x)=1$, and $S(x)$ is a diagonal matrix with non-negative elements, for all $a=1,\dots, A$ and $x\in E$. 

  A particular decomposition is given by
  \begin{equation}\label{Snormdec}
    S_{aa}(x) = K_{aa}(x,x)^\frac{1}{2}
  \end{equation}
  and  
  \begin{equation}\label{Rnormdec}
    R_{ab}(x,y)=\begin{cases} 1,&\text{if $a=b$ and $x=y$,}\\ S_{aa}(x)^{-1} K_{ab}(x,y) S_{bb}(y)^{-1}, &  \text{if $S_{aa}(x)>0$ and $S_{bb}(y)>0$,}\end{cases}
  \end{equation}
  and we set 
  \begin{equation}\label{Rnormdec0}
    R_{ab}(x,y)=0\quad \text{otherwise.}
  \end{equation}  
  On the other hand, any such decomposition necessarily satisfies \eqref{Snormdec} and \eqref{Rnormdec}.\footnote{Property \eqref{Rnormdec0} does not necessarily hold. Indeed, consider the finite set $E=\{x_1,x_2\}$ and $A=1$ and suppose that $K(x_1,x_1)=1$ and $K(x_1,x_2)=K(x_2,x_2)=0$. Then $R(x_1,x_1)=1$, $R(x_1,x_2)=1/2$ and $R(x_2,x_2)=1$ is a normalized kernel satisfying the decomposition \eqref{eqnormdec}, but not \eqref{Rnormdec0}.}
\end{lemma} 

\begin{proof}
  Necessity of \eqref{Snormdec} and \eqref{Rnormdec} follows by inspection.

  It remains to prove that $R_{ab}(x,y)$ given by \eqref{Rnormdec} and \eqref{Rnormdec0} defines a $\mathbb{R}^{A\times A}$-valued kernel. It is readily verified that $R_{ab}(x,y)=R_{ba}(y,x)$, which proves \eqref{Ktop}. As for \eqref{PSD}, we define the index set $\Ical_0=\{(a,i)\mid S_{aa}(x_i)=0\}$ and its complement $\Ical_1=\Ical_0^c$. Now let $v_1,\dots,v_n\in\R^A$, and define $w_i\in\R^A$ by $w_{ia}=v_{ia}S_{aa}(x_i)^{-1}$ for $(a,i)\in\Ical_1$ and $w_{ia}=0$ otherwise. Then we have
  \begin{align*}
    \sum_{i,j=1}^n v_i^\top R(x_i,x_j)v_j &=  \sum_{a,b=1}^A \sum_{i,j=1}^n v_{ia}  R_{ab}(x_i,x_j)v_{jb}=\sum_{(a,i)\in\Ical_0} v_{ia}^2+\sum_{(a,i),(b,j)\in\Ical_1}v_{ia}  R_{ab}(x_i,x_j)v_{jb}\\
                                          &\ge \sum_{(a,i),(b,j)\in\Ical_1} w_{ia}  K_{ab}(x_i,x_j) w_{jb} =  \sum_{i,j=1}^n w_{i}^\top  K(x_i,x_j)w_{j}\ge 0
  \end{align*}
  by the kernel property \eqref{PSD} of $K$. This completes the proof.
\end{proof}

In the special case of separable kernels, Lemma \ref{lemnormdec} extends as follows.

\begin{lemma}
  \label{lem:normalizedsep}
  Let $K(x,y) = Bk(x,y)$ be a separable kernel. Then the normalized kernel given by \eqref{Rnormdec} and \eqref{Rnormdec0} is separable of the form $R(x,y)=C \rho(x,y)$ for the symmetric positive semi-definite matrix $C$ given by
  \[ C_{ab}= \begin{cases} 1,&\text{if $a=b$,}\\ B_{aa}^{-\frac{1}{2}} B_{ab}B_{bb}^{-\frac{1}{2}} ,&\text{if $B_{aa}>0$ and $B_{bb}>0$,}\end{cases} \]
  and we set $C_{ab}=0$ otherwise
  and the scalar kernel $\rho$ given by \[\rho(x,y)=\begin{cases} 1,&\text{if $x=y$,}\\ k(x,x)^{-\frac{1}{2}}k(x,y)k(y,y)^{-\frac{1}{2}} ,&\text{if $k(x,x)>0$ and $k(y,y)>0$,}\end{cases}\] 
  and we set $\rho(x,y)=0$ otherwise. In particular, $C$ and $\rho$ are normalized in the sense that $C_{aa}=1$ and $\rho(x,x)=1$, for all $a=1,\dots, A$ and $x\in E$.
\end{lemma}

\begin{proof}
  It is enough to show that $C$ is a symmetric positive semi-definite matrix and $\rho$ a scalar kernel. This can both be proved using similar arguments as in the proof of Lemma~\ref{lemnormdec}.
\end{proof}

\section{Proofs}

This appendix provides the proofs of the results stated in the main text, based on the foundational material presented in Appendix~\ref{secRKHS}.

\subsection{Proof of Theorem \ref{thm:vect_krr}}

Let $S$ be the sampling operator as in equation \eqref{eq:sampling_op}.
  For any $m \in \{1,\dots,M\}$ define $a(m),i(m)$ such that $\bm C_{m} = (\dots, C_{a(m),i(m)}, \dots)$ is the $m$-th row of $\bm C$, and $\bm P_m=P_{a(m),i(m)}$ is the $m$-th component of $\bm P$, and $\bm\omega_m = \omega_{a(m),i(m)}$ the corresponding weight. Then the weighted mean-squared pricing error can be written as
  \[ \sum_{m=1}^M \bm \omega_m ( \bm P_m-\bm C_m \vect( p^\top(\bm x)) -\bm C_m Sh)^2 .\]
Similarly for the constraints, where $\bm \omega_m=\infty$. 

It then follows that the solution of the KR problem must lie in the orthogonal complement of the null space of $\bm C S$. That is, $h= S^\ast\bm C^\top q$, for some $q\in\R^M$. The rest of the proof now follows as in the scalar case \cite[Theorem A.1]{Filipovic2022}, using Lemma~\ref{lemadjS} below. This completes the proof of Theorem \ref{thm:vect_krr}.

\begin{lemma}\label{lemadjS}
  Define the sampling operator $S:\Hcal \to \R^{AN}$ by 
  \begin{equation}
  \label{eq:sampling_op}
Sh = \vect(h^\top(\bm x)).      
  \end{equation}
  The adjoint $S^\ast:\R^{AN}\to \Hcal$ is given by
  \begin{equation}\label{eqSast}
    S^\ast v = \sum_{j=1}^N K(\cdot,x_j) V_j^\top
  \end{equation}
  where $V_j$ is the $j$-th row of the matrix $V\in\R^{N\times A}$ with $\vect(V)=v$. Moreover, $\bm K$ is the matrix representation of the linear operator $SS^\ast:\R^{AN}\to\R^{AN}$ in the standard Euclidean basis of $\R^{AN}$.
\end{lemma}

\begin{proof}[Proof of Lemma \ref{lemadjS}]
  Let $v\in \R^{AN}$ and $V\in \R^{N\times A}$ its matricization such that $\vect(V)=v$. Then
  \[    \langle Sh,v\rangle_{\R^{AN}} = \sum_{j=1}^N\sum_{a=1}^A h_a(x_j) V_{ja} = \sum_{j=1}^N V_j  h(x_j) = \sum_{j=1}^N \langle  h , K(\cdot, x_j) V_j^\top\rangle_\Hcal,\]
  which proves \eqref{eqSast}. In coordinates, \eqref{eqSast} reads as
  \[ S^\ast v = \sum_{b=1}^A\sum_{j=1}^N  \left(K_{1b}(\cdot,x_j),K_{2b}(\cdot,x_j),\dots,K_{Ab}(\cdot,x_j)\right)^\top V_{jb},\]
  and thus we obtain
  \[ SS^\ast v = \sum_{b=1}^A \sum_{j=1}^N  \vect\left(K_{1b}(\bm x,x_j),K_{2b}(\bm x,x_j),\dots,K_{Ab}(\bm x,x_j)\right) V_{jb} = \bm K v,\]
  as desired.
\end{proof}

\subsection{Proof of Theorem \ref{thm_vkr}}

According to Theorem \ref{thmRHKSsepNEW}\ref{thmRHKSsepNEW4} it is enough to construct a symmetric positive definite matrix $Q$ such that $\gamma_a=  \sum_{b=1}^A Q_{ab}$ and $Q_{ab}=-\Theta_{ab}$ for $a<b$. Therefore, we parameterize $Q$ by the $A(A-1)/2$ spread smoothness parameters $\Theta_{ab} \ge 0$, as defined in \eqref{eqgraphreg}.

By construction, the matrix $Q$ is strictly diagonally dominant, $Q_{aa} > \sum_{b\neq a} |Q_{ab}|$, for all $a$, and hence positive definite, see \cite[Theorem 6.1.10]{horn_johnson_2012}. Hence $B=Q^{-1}$ is symmetric and positive definite leading to a valid separable kernel. Theorem \ref{thmRHKSsepNEW} implies that the norm of the vector-valued RKHS $\Hcal$ with separable kernel $K(x,y) = Bk(x,y)$ is given by \eqref{eq:graph_norm}. Theorem \ref{thmRHKSsepNEW} also implies that the optimization problem \eqref{eq:multi_curve_motivated} over the product space $(\Hcal_k)^A$ is equivalent to the KR problem \eqref{eq:krr_obj} with norm \eqref{eq:graph_norm} for $\lambda=1$.

\begin{remark}
  The matrix $Q$ in \eqref{eqgraphreg} is strictly diagonally dominant, by construction. This is sufficient for $Q$ being positive definite. However, not every symmetric positive definite matrix is strictly diagonally dominant. An example is given by
  \begin{equation*}
    Q=
    \begin{pmatrix}
      4 & q \\ q & 1
    \end{pmatrix},
  \end{equation*}
  for any $1<q<2$. Indeed, the characteristic polynomial is  
  $(4-\lambda)(1-\lambda)-q^2 = \lambda^2-5\lambda + 4-q^2$. Hence the eigenvalues of $Q$ are positive, $\lambda_{1,2}=\frac{5\pm \sqrt{25-4(4-q^2)}}{2}>0$,
  and $Q$ is positive definite. However, $Q$ is not diagonally dominant, as $Q_{22}=1 < q = Q_{21}$. In that sense, specification \eqref{eq:graph_norm} is a special case of a vector-valued RKHS with separable kernel as discussed in Theorem \ref{thmRHKSsepNEW}    
\end{remark}

\subsection{Proof of Lemma \ref{lemYR}}
Under the assumption of the lemma, we have after multiplication with $\e^{T_0 Y}$
\[ 0 = \Delta R \sum_{j=1}^n \e^{-\Delta Y j} + \e^{-\Delta Y n} -1 = \Delta R \frac{q}{1-q}(1-q^n) - (1-q^n),\]
where we write $q= \e^{-\Delta Y}$. Therefore $\Delta R = \frac{1-q}{q}$, which proves the claim.

\section{Arbitrage-Free Pricing Framework}
\label{appendix:na}

In this appendix, we place the discounted cash flow equation~\eqref{eq:pricing_eq} within an arbitrage-free pricing framework, following standard principles of asset pricing theory (see, e.g., \cite{bjo_09}). 

Let $(\Omega, \Fcal, \Q)$ be a probability space equipped with a filtration $(\Fcal_t)_{t \ge 0}$ representing the flow of market information. All processes are assumed to be adapted to this filtration. The pricing measure $\Q$ is risk-neutral with respect to a numeraire $B(t)$, interpreted as the money market account, satisfying $B(0) = 1$ and accruing at the overnight RFR. The present value at time $0$ of an $\Fcal_T$-measurable cash flow $Z$ paid at time $T > 0$ is
\begin{equation}\label{eqnPVformula}
    PV_Z = \E_\Q\bigg[\frac{Z}{B(T)}\bigg] = \E_{\Q^T}[Z]\, g_0(T),
\end{equation}
where $g_0(T) = \E_\Q\big[\frac{1}{B(T)}\big]$ is the price of a risk-free discount bond maturing at $T$, and $\Q^T$ denotes the $T$-forward measure defined via the Radon--Nikodym derivative $\frac{d\Q^T}{d\Q} = \frac{1}{g_0(T)\, B(T)}$.

\subsection{Non-Defaultable Bonds}

Bonds issued by highly rated sovereigns, such as US\ Treasuries or German government bonds, are typically regarded as non-defaultable (or risk-free). The discounted cash flow equation~\eqref{eq:pricing_eq} applies directly with $g_a = g_0$ for such a bond paying nominal coupons $c_1,\dots,c_n$ at dates $0 < T_1 < \dots < T_n$ and the notional of one at the maturity $T_n$. 

\subsection{Defaultable Bonds}\label{ssec_defaultbond}

Defaultable (or credit-risky) bonds include corporate debt and sovereign debt issued by less creditworthy countries. These instruments generally trade at a spread over the risk-free curve to reflect credit risk. Let $\tau$ denote the default time (which is a stopping time). Under the widely used recovery-of-treasury assumption (see~\cite{Jarrow1995-ss}), the cash flow at $T_i$ is modeled as
\[
Z_i = c_i \, \mathbf{1}_{\{\tau > T_i\}} + c_i \, \delta_i \, \mathbf{1}_{\{\tau \le T_i\}},
\]
where $\delta_i \in [0,1)$ is a deterministic recovery rate. Applying~\eqref{eqnPVformula}, we obtain
\[
PV_{Z_i} = \E_{\Q^T}[Z_i]\, g_0(T_i) = c_i \big( \Q^{T_i}[\tau > T_i] + \delta_i\, \Q^{T_i}[\tau \le T_i] \big) g_0(T_i),
\]
which motivates the effective discount factor
\begin{equation}\label{defbona0}
    g_a(T_i) = \big( \Q^{T_i}[\tau > T_i] + \delta_i\, \Q^{T_i}[\tau \le T_i] \big) g_0(T_i).
\end{equation}

Defaultable bonds are typically grouped by credit rating. Assuming all bonds within a given rating class~$a$ share the same default distribution and recovery profile, the class admits a common discount curve $g_a(x)$, and the discounted cash flow equation~\eqref{eq:pricing_eq} applies.

\subsection{RFR-Based Swaps}\label{ssec_OIS}

An RFR-based swap is an interest-rate swap whose floating leg is linked to the money market account \( B(t) \), which accrues at the RFR, such as SOFR in the United States or SARON in Switzerland, see \cite{FedSofr, SixSaron}. Under the no-arbitrage assumption, RFR-based swap contracts should be priced using the same discount curve \( g_0 \) as creditworthy government bonds denominated in the same currency. In practice, however, a swap--government bond spread is observed. This spread arises due to market frictions and regulatory effects, and lies outside the scope of our simple arbitrage-free pricing framework, see, e.g., \cite{wu_jar_24}.

As for the fixed leg, let \( T_0 < T_1 < \dots < T_n \) denote the payment dates, with notional normalized to one. For a given annualized swap rate \( R \), the fixed cash flow at time \( T_i \) is \( \Delta  R \) with $\Delta = T_i-T_{i-1}$. By \eqref{eqnPVformula}, the present value of the fixed leg is  
\begin{equation*}\label{eqnOISfixed}
    PV_\text{fixed} = \Delta R\sum_{i=1}^n  \, g_0(T_i).
\end{equation*}

Let \(  T_0 = t_0 < \dots < t_m = T_n \) denote the reset and payment dates of the RFR floating leg, again with notional normalized to one. The floating cash flow at time \( t_i > 0 \) corresponds to the simple return of the money market account over the accrual period \( [t_{i-1}, t_i] \), given by $\frac{B(t_i)}{B(t_{i-1})} - 1$. Using \eqref{eqnPVformula} and observing the telescoping structure of the discounted cash flows, we obtain $\sum_{i=1}^m \frac{1}{B(t_i)}\big( \frac{B(t_i)}{B(t_{i-1})} - 1 \big) = \frac{1}{B(T_0)} - \frac{1}{B(T_n)}$,
from which the present value of the RFR floating leg follows as  
\begin{equation}\label{eqnOISfloat}
    PV_\text{RFR–floating} = g_0(T_0) - g_0(T_n).
\end{equation}

Although the above specification, where floating cash flows are "fixed in arrears," has become the standard, see, e.g., \cite{icma2020rfrguide, risk2021sofr}, an alternative is to define the floating rate over \( [t_{i-1}, t_i] \) as the simple return on a discount bond, $R_\text{term}(t_{i-1}, t_i) = \frac{1}{g_0(t_{i-1}, t_i)} - 1$. Here, with a slight abuse of notation, we denote by \( g_0(t, T) = \E_\Q\big[ \frac{1}{B(T)} \mid \Fcal_t \big] \) the time-\( t \) value of a risk-free discount bond maturing at \( T \), such that \( g_0(x) = g_0(0,x) \). Under this alternative specification, the present value of the floating leg remains given by \eqref{eqnOISfloat}, which follows directly as a simple consequence of the arbitrage-free pricing formula \eqref{eqnPVformula}.

\subsection{IBOR Swaps}

Interest-rate swaps whose floating leg is tied to an interbank loan term rate (IBOR) reflect credit and liquidity risk, which we model by adding a spread to the floating cash flows. For example, EURIBOR can be viewed as the sum of the risk-free ESTR and a credit spread capturing interbank risk.\footnote{Strictly speaking, ESTR is not secured, unlike SOFR. However, as an overnight rate, its credit risk is considered negligible, and we treat it as risk-free for our purposes.}

Formally, using the same tenor structures for the floating and fixed legs as in Subsection~\ref{ssec_OIS}, the floating cash flow of an IBOR swap at time \( t_i \) is given by $R_\text{term}(t_{i-1}, t_i) + S(t_{i-1}, t_i)$, where \( S(t_{i-1}, t_i) \) denotes a spread that reflects the credit and liquidity risk of lending in the interbank market over the period \( [t_{i-1}, t_i] \). The present value of the IBOR swap’s floating leg is then  
\begin{equation}\label{eqnIBORfloat}
    PV_\text{IBOR–floating} = g_0(T_0) - g_0(T_n) + \sum_{i=1}^n \E_{\Q^{t_i}}[S(t_{i-1}, t_i)]\, g_0(t_i).
\end{equation}

As in the case of defaultable bonds discussed in Subsection~\ref{ssec_defaultbond}, we classify IBOR swaps according to the length of the accrual period (tenor) of the floating leg, such as quarterly, semiannual, or annual. We assume that all IBOR swaps within a given tenor class \( a \) share the same spread structure, which gives rise to a common discount curve \( g_a(x) \). This curve is determined from the discounted cash flow equation~\eqref{eq:pricing_eq}, in conjunction with the expressions for the floating and fixed cash flows, resulting from \eqref{sswapcf} and \eqref{fswapcf}, respectively.

For positive spreads \( S(t_{i-1}, t_i) > 0 \), the discount curve implied by the IBOR swap is strictly below the RFR-based swap curve, that is, \( g_a(x) < g_0(x) \). However, as seen from \eqref{eqnIBORfloat}, this relationship is not as explicit as in the recovery-of-treasury model for defaultable bonds, as given in~\eqref{defbona0}.

\subsection{Cross-Currency Swaps}

We are considering a standard floating--floating XCCY swap. The tenor structure of the cash flows is given by $0\le t_0 <t_1<\dots <t_m$. The XCCY consists of two legs, leg $a$ and leg $b$. Leg $b$ is treated as the liquid leg. Thus, the basis spread $s$ is added to leg $a$ which has a normalized notional of $1$. The initial notional of leg $b$ is set to the spot exchange rate, \( X_{ab}(t_0) \). The MTM feature is sometimes applied to leg~$b$. 

We now show that, when present, the MTM adjustments do not affect the present value of leg~$b$. According to Clarus Financial Technology,\footnote{Clarus FT is a data and analytics provider focused on OTC derivatives markets. See \cite{clarusft2017} for a discussion of MTM mechanics in cross-currency swaps.} the floating cash flow \( Z_i \) at each payment date \( t_i > 0 \) consists of the simple return on the money market account applied to the MTM notional over the accrual period \( [t_{i-1}, t_i] \), minus the change in MTM notionals over that period, and plus the MTM notional at maturity if \( t_i = t_m \). Formally, this gives  
\[
Z_i = X_{ab}(t_{i-1}) \bigg( \frac{B_b(t_i)}{B_b(t_{i-1})} - 1 \bigg) - \left( X_{ab}(t_i) - X_{ab}(t_{i-1}) \right) + X_{ab}(T) \, 1_{t_i = t_m},
\]
where \( X_{ab}(t) \) denotes the MTM notional in currency \( b \) at time \( t \), and \( B_b(t) \) is the corresponding money market account.

Discounting each cash flow by the money market account and simplifying the telescoping sum yields  
\[
\sum_{i=1}^m \frac{Z_i}{B_b(t_i)} = \sum_{i=1}^m \left( \frac{X_{ab}(t_{i-1})}{B_b(t_{i-1})} - \frac{X_{ab}(t_i)}{B_b(t_i)} \right) + \frac{X_{ab}(t_m)}{B_b(t_m)} = X_{ab}(t_0),
\]
which is known (deterministic) at time $t_0$ and equal to the initial notional. Hence, the present value of leg \( b \) is given by $X_{ab}(t_0)$, as in the case without MTM. This demonstrates that MTM adjustments, while relevant for risk management, do not affect the arbitrage-free valuation of the liquid leg.

\section{Additional Yield and Forward Curves}

This appendix complements Section~\ref{sec:ill_yield_and_fwd_curves} by additional example days shown in Figures~\ref{fig:example_day_2020}--\ref{fig:example_day_2023}.

\label{sec:additional_results}
\begin{figure}[ht!]
  \centering
       \tcapfig{Example day 2020-06-15}
  \begin{subfigure}[t]{0.49\linewidth}
    \centering
    \includegraphics[width=\linewidth]{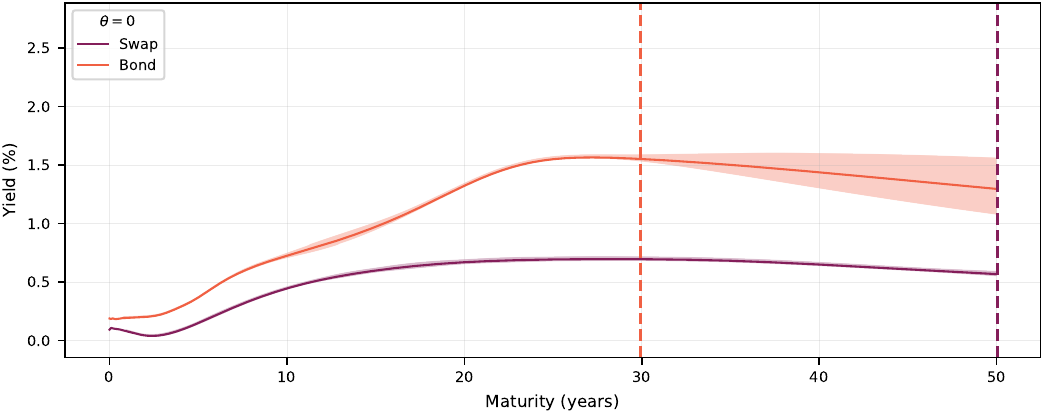}
  \end{subfigure}\hfill
  \begin{subfigure}[t]{0.49\linewidth}
    \centering
    \includegraphics[width=\linewidth]{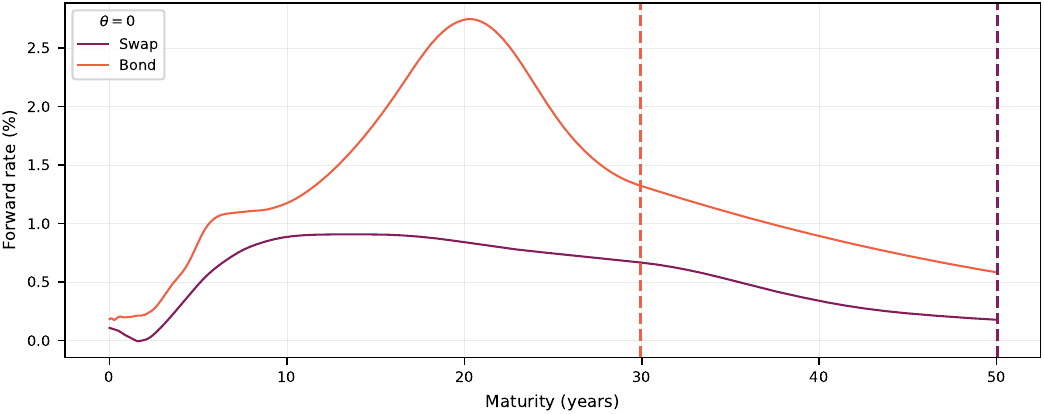}
  \end{subfigure}\vspace{0.3em}
    \begin{subfigure}[t]{0.49\linewidth}
    \centering
    \includegraphics[width=\linewidth]{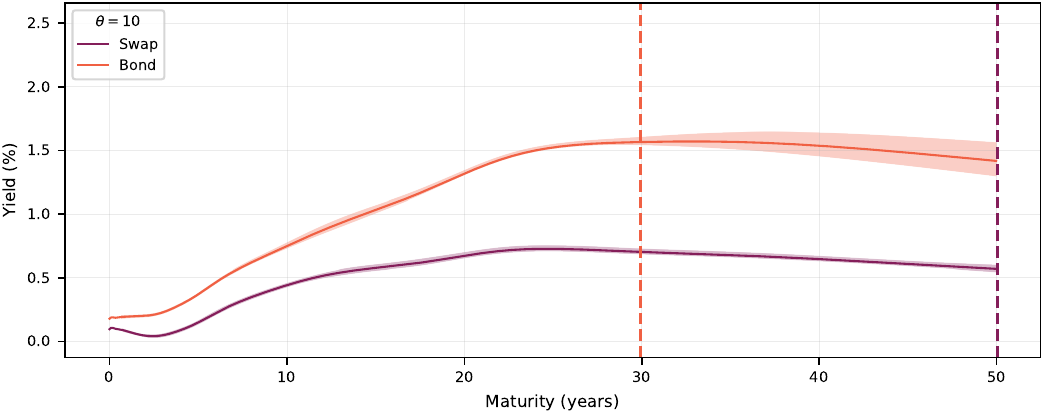}
  \end{subfigure}\hfill
  \begin{subfigure}[t]{0.49\linewidth}
    \centering
    \includegraphics[width=\linewidth]{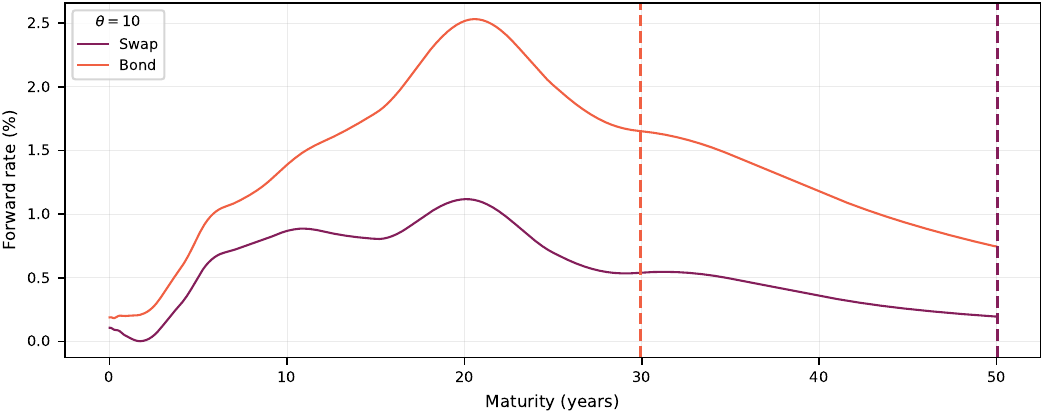}
  \end{subfigure}\vspace{0.3em}
    \begin{subfigure}[t]{0.49\linewidth}
    \centering
    \includegraphics[width=\linewidth]{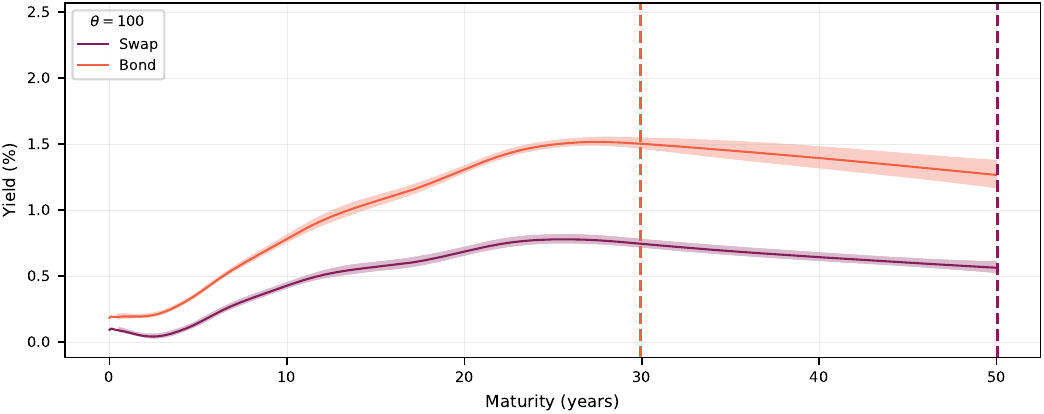}
  \end{subfigure}\hfill
  \begin{subfigure}[t]{0.49\linewidth}
    \centering
    \includegraphics[width=\linewidth]{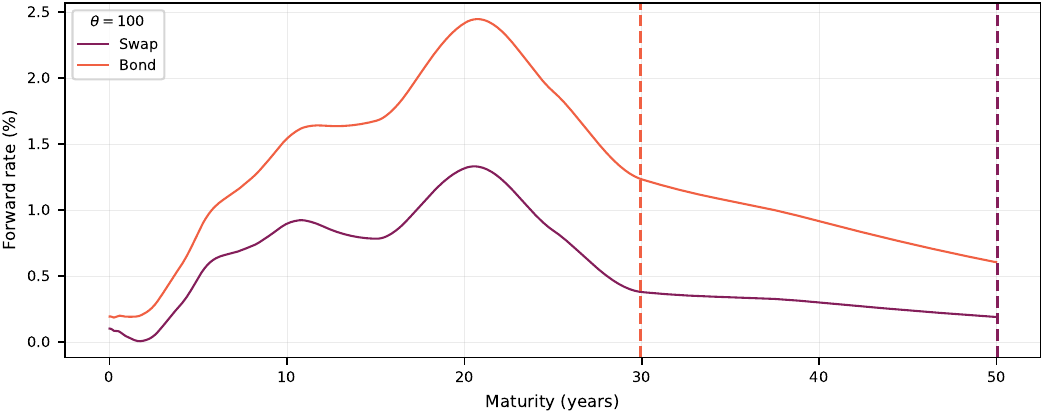}
  \end{subfigure}\vspace{0.3em}
    \begin{subfigure}[t]{0.49\linewidth}
    \centering
    \includegraphics[width=\linewidth]{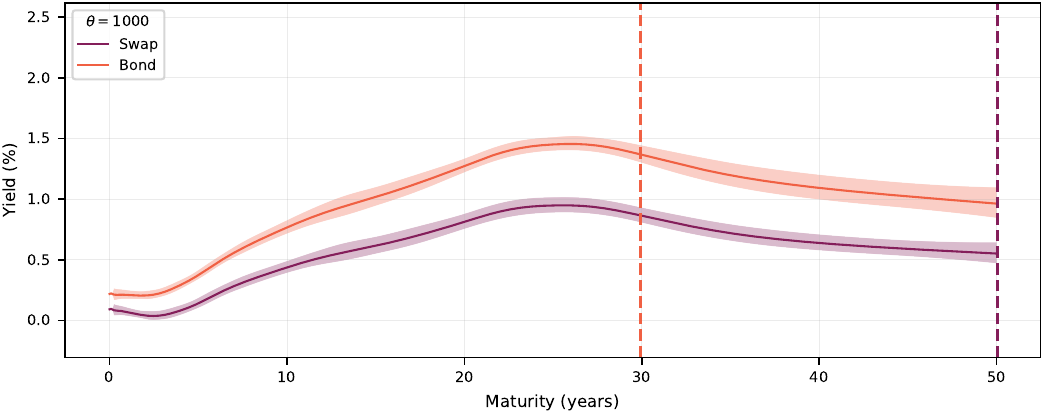}
  \end{subfigure}\hfill
  \begin{subfigure}[t]{0.49\linewidth}
    \centering
    \includegraphics[width=\linewidth]{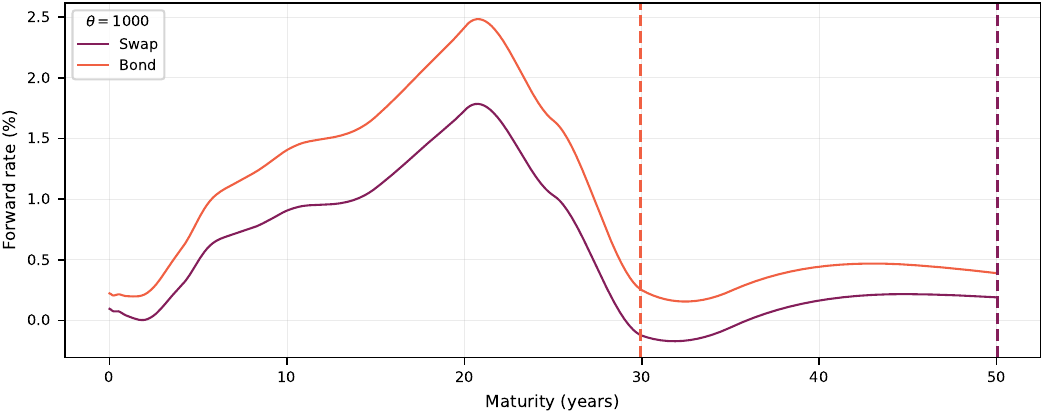}
  \end{subfigure}\vspace{0.3em}
        \bnotefig{This figure shows the resulting yield curves for various $\theta$ on the left and respective forward rate curves on the right on 2020-06-15. In all panels, the vertical dashed lines indicate the longest available data point in the respective product class. The shaded areas show the $3\sigma$ confidence bands derived from the Gaussian process view and are capped at $\pm 2\%$. All values are in \%} 
   \label{fig:example_day_2020}
\end{figure}
\begin{figure}[ht!]
  \centering
       \tcapfig{Example day 2021-06-15}
  \begin{subfigure}[t]{0.49\linewidth}
    \centering
    \includegraphics[width=\linewidth]{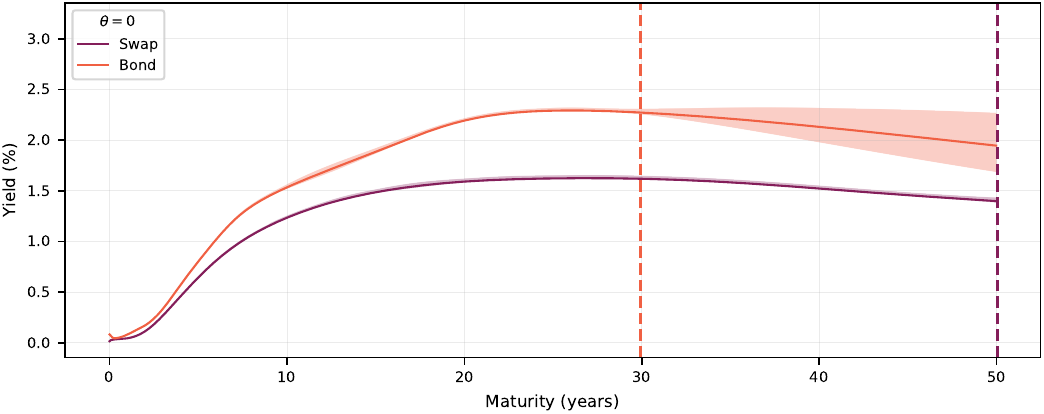}
  \end{subfigure}\hfill
  \begin{subfigure}[t]{0.49\linewidth}
    \centering
    \includegraphics[width=\linewidth]{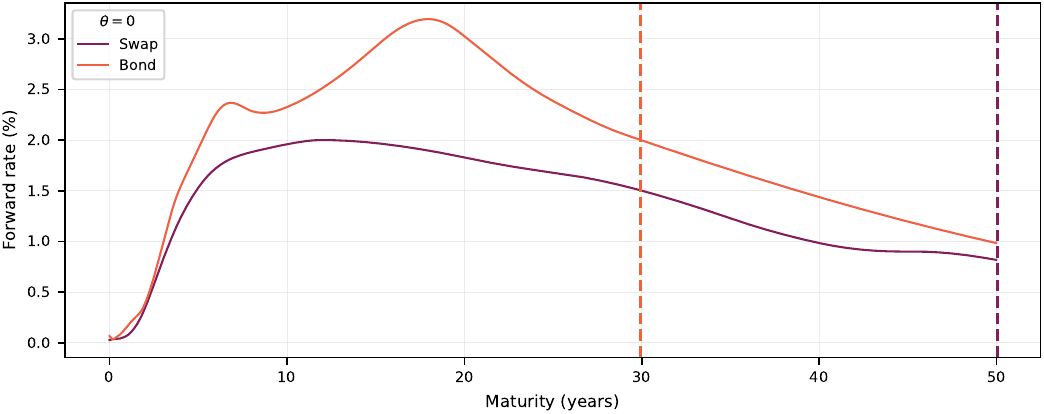}
  \end{subfigure}\vspace{0.3em}
    \begin{subfigure}[t]{0.49\linewidth}
    \centering
    \includegraphics[width=\linewidth]{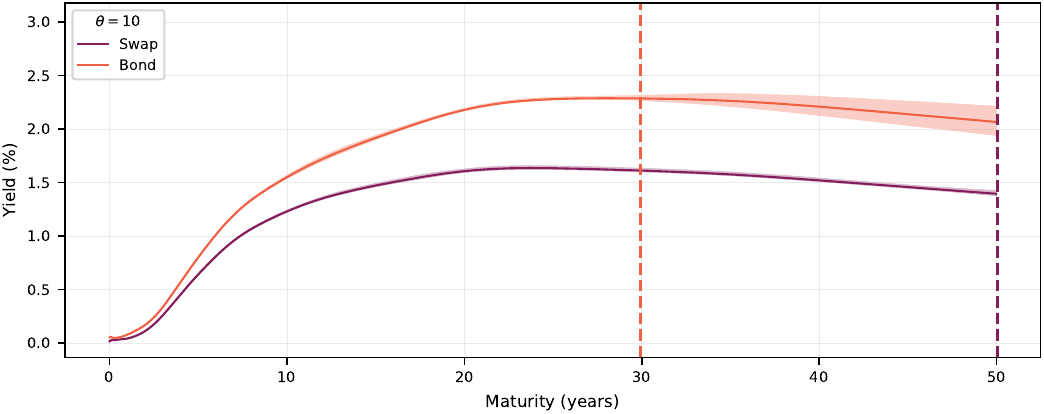}
  \end{subfigure}\hfill
  \begin{subfigure}[t]{0.49\linewidth}
    \centering
    \includegraphics[width=\linewidth]{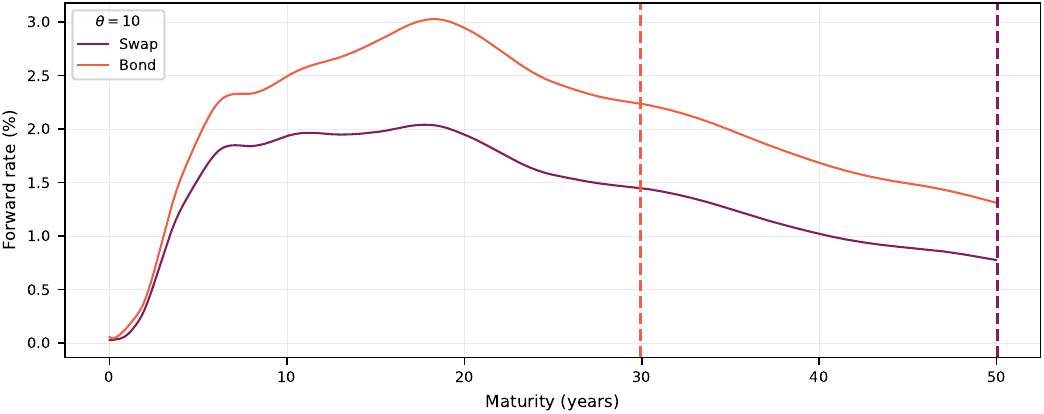}
  \end{subfigure}\vspace{0.3em}
    \begin{subfigure}[t]{0.49\linewidth}
    \centering
    \includegraphics[width=\linewidth]{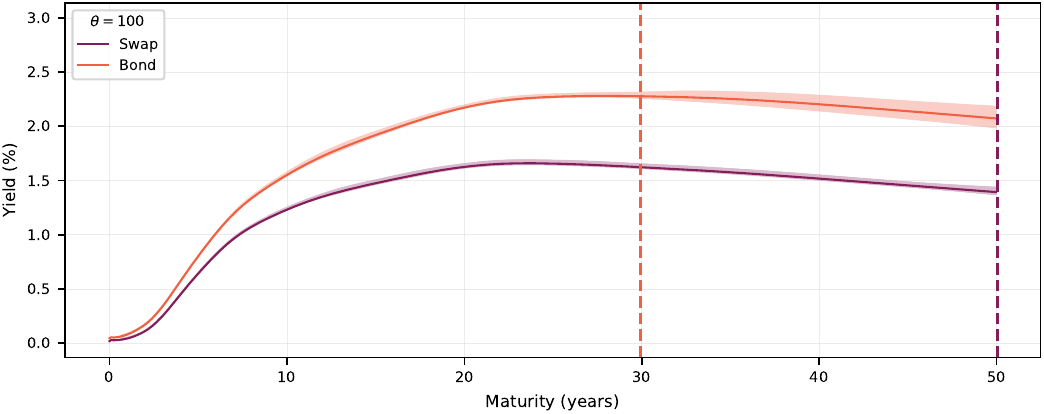}
  \end{subfigure}\hfill
  \begin{subfigure}[t]{0.49\linewidth}
    \centering
    \includegraphics[width=\linewidth]{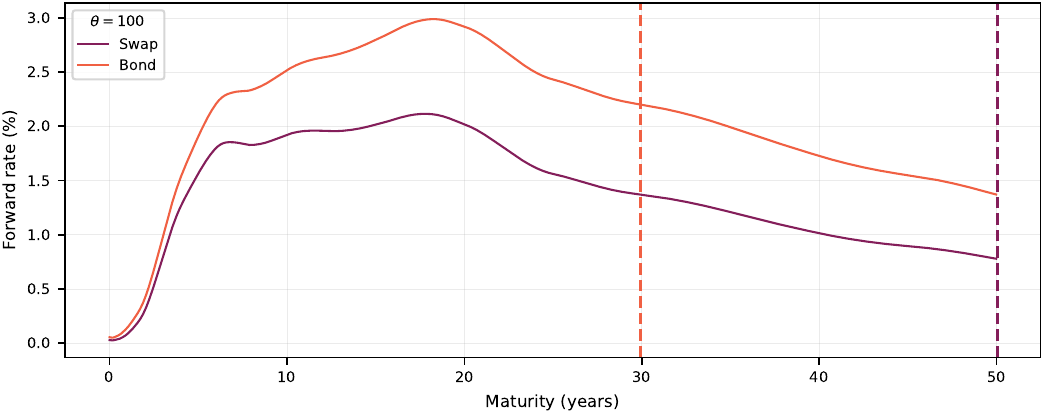}
  \end{subfigure}\vspace{0.3em}
    \begin{subfigure}[t]{0.49\linewidth}
    \centering
    \includegraphics[width=\linewidth]{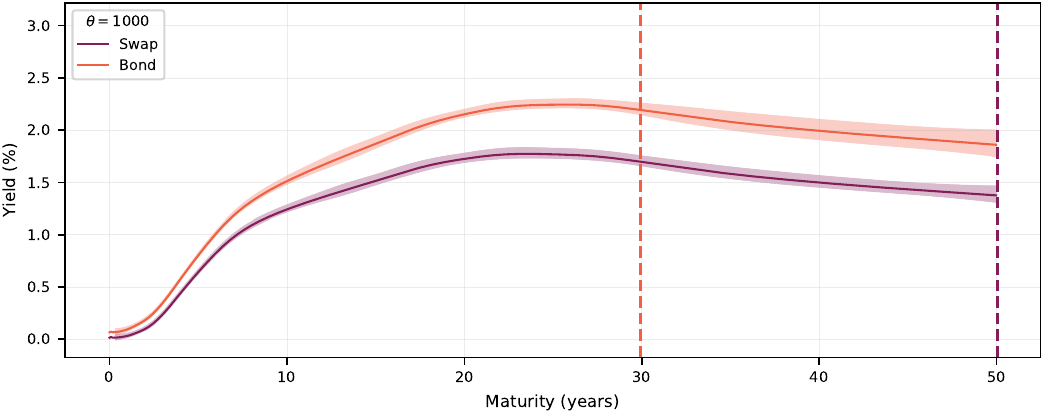}
  \end{subfigure}\hfill
  \begin{subfigure}[t]{0.49\linewidth}
    \centering
    \includegraphics[width=\linewidth]{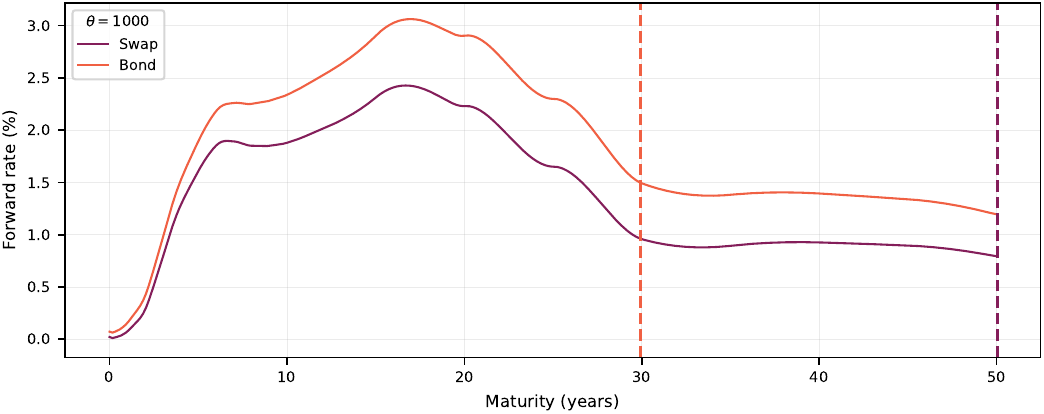}
  \end{subfigure}\vspace{0.3em}
        \bnotefig{This figure shows the resulting yield curves for various $\theta$ on the left and respective forward rate curves on the right on 2021-06-15. In all panels, the vertical dashed lines indicate the longest available data point in the respective product class. The shaded areas show the $3\sigma$ confidence bands derived from the Gaussian process view and are capped at $\pm 2\%$. All values are in \%} 
   \label{fig:example_day_2021}
\end{figure}
\begin{figure}[ht!]
  \centering
       \tcapfig{Example day 2022-06-15}
  \begin{subfigure}[t]{0.49\linewidth}
    \centering
    \includegraphics[width=\linewidth]{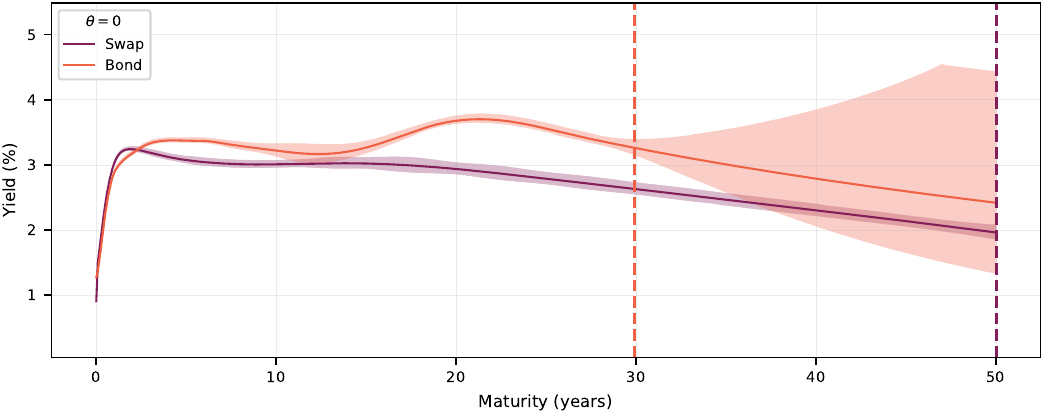}
  \end{subfigure}\hfill
  \begin{subfigure}[t]{0.49\linewidth}
    \centering
    \includegraphics[width=\linewidth]{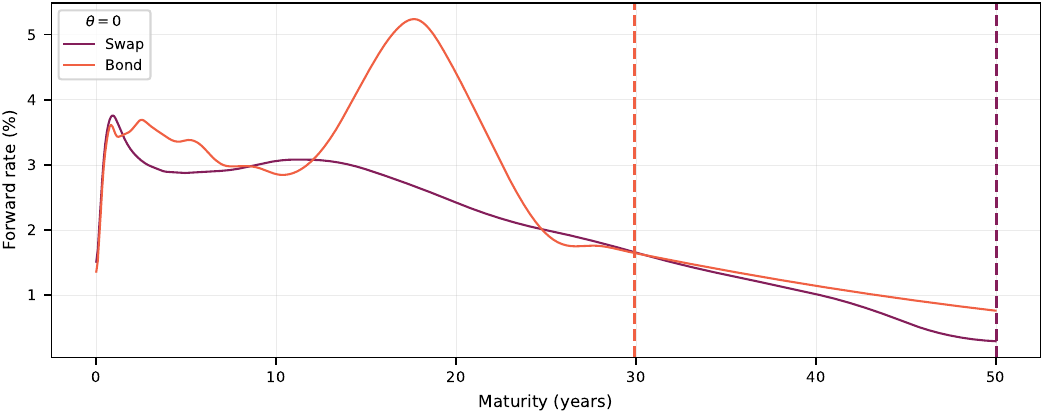}
  \end{subfigure}\vspace{0.3em}
    \begin{subfigure}[t]{0.49\linewidth}
    \centering
    \includegraphics[width=\linewidth]{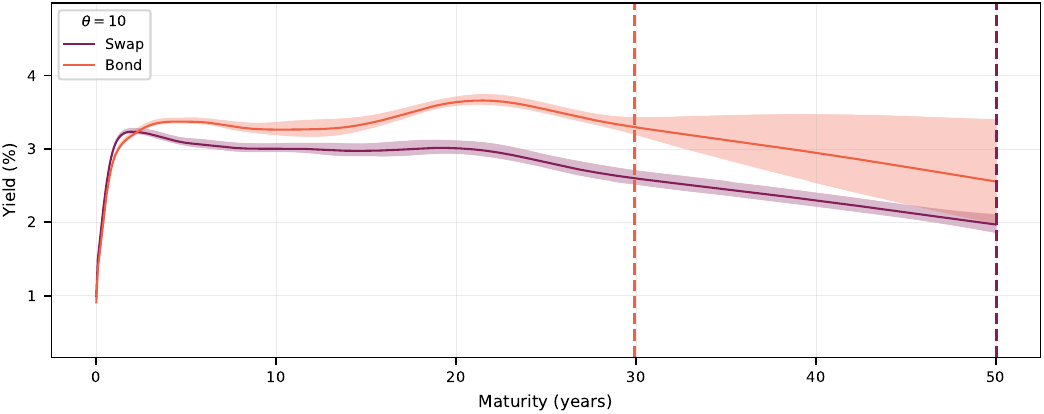}
  \end{subfigure}\hfill
  \begin{subfigure}[t]{0.49\linewidth}
    \centering
    \includegraphics[width=\linewidth]{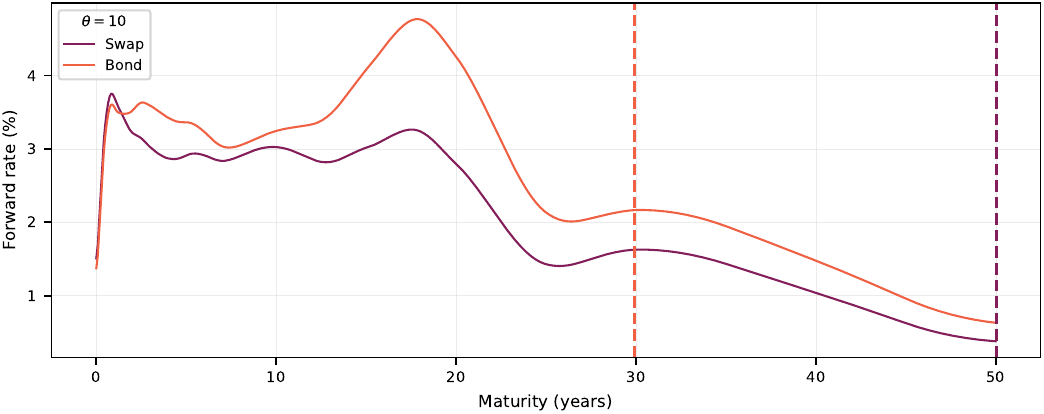}
  \end{subfigure}\vspace{0.3em}
    \begin{subfigure}[t]{0.49\linewidth}
    \centering
    \includegraphics[width=\linewidth]{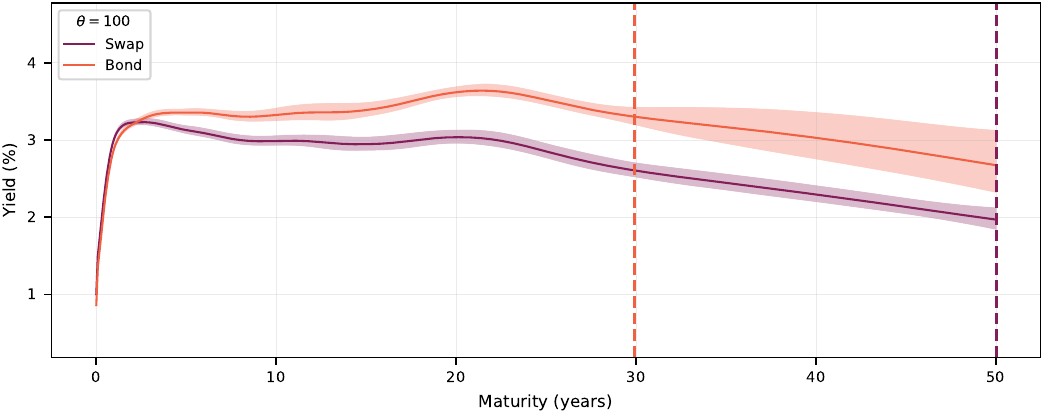}
  \end{subfigure}\hfill
  \begin{subfigure}[t]{0.49\linewidth}
    \centering
    \includegraphics[width=\linewidth]{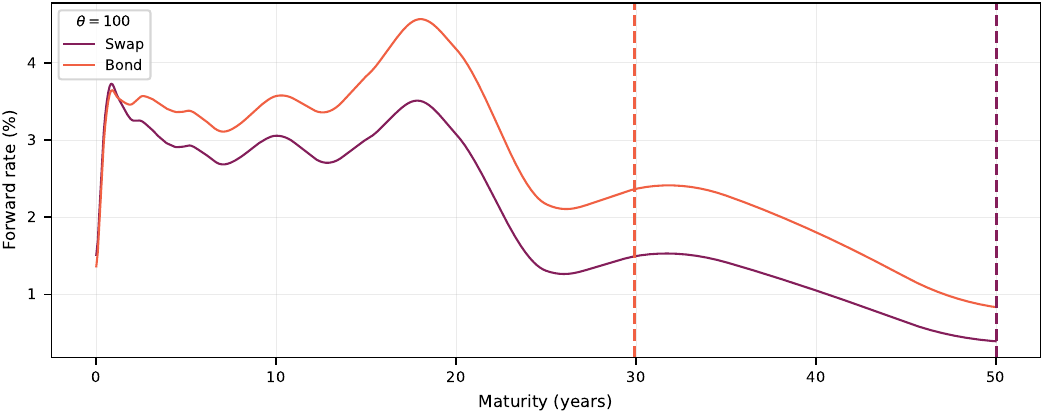}
  \end{subfigure}\vspace{0.3em}
    \begin{subfigure}[t]{0.49\linewidth}
    \centering
    \includegraphics[width=\linewidth]{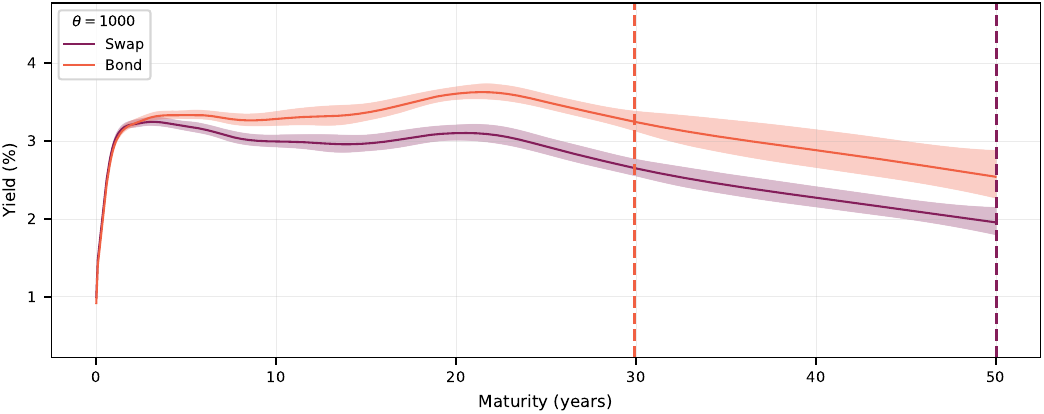}
  \end{subfigure}\hfill
  \begin{subfigure}[t]{0.49\linewidth}
    \centering
    \includegraphics[width=\linewidth]{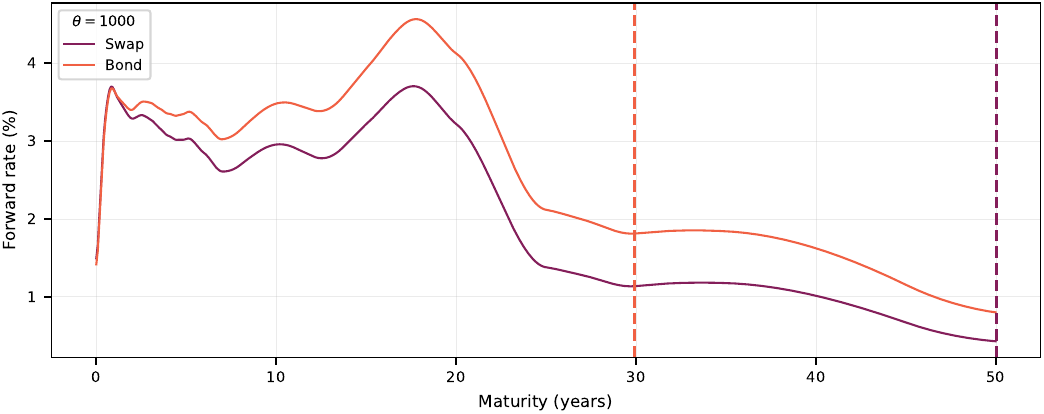}
  \end{subfigure}\vspace{0.3em}
        \bnotefig{This figure shows the resulting yield curves for various $\theta$ on the left and respective forward rate curves on the right on 2022-06-15. In all panels, the vertical dashed lines indicate the longest available data point in the respective product class. The shaded areas show the $3\sigma$ confidence bands derived from the Gaussian process view and are capped at $\pm 2\%$. All values are in \%} 
   \label{fig:example_day_2022}
\end{figure}
\begin{figure}[ht!]
  \centering
       \tcapfig{Example day 2023-06-15}
  \begin{subfigure}[t]{0.49\linewidth}
    \centering
    \includegraphics[width=\linewidth]{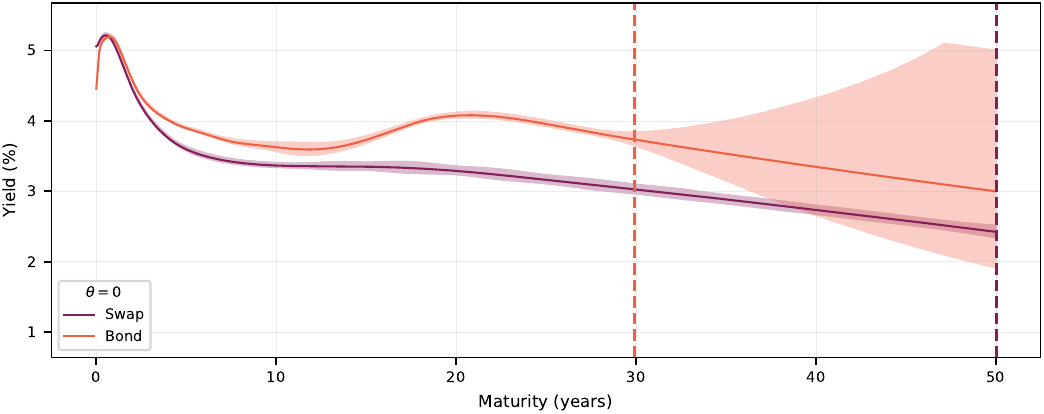}
  \end{subfigure}\hfill
  \begin{subfigure}[t]{0.49\linewidth}
    \centering
    \includegraphics[width=\linewidth]{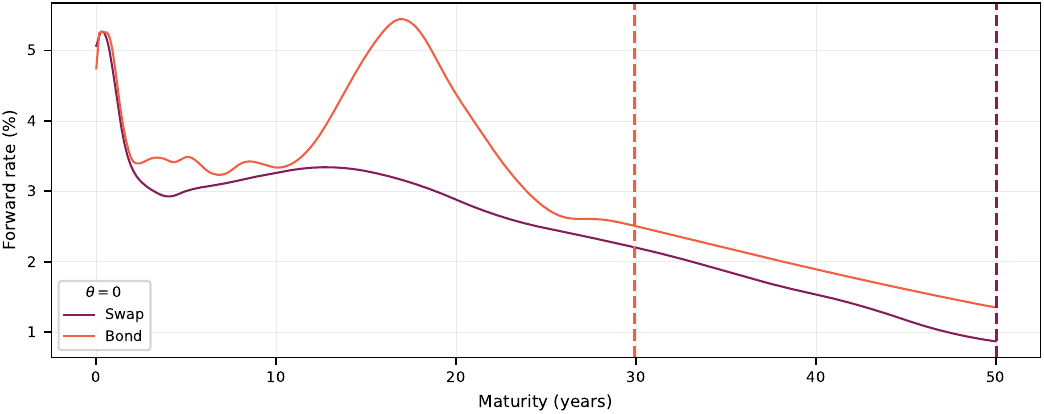}
  \end{subfigure}\vspace{0.3em}
    \begin{subfigure}[t]{0.49\linewidth}
    \centering
    \includegraphics[width=\linewidth]{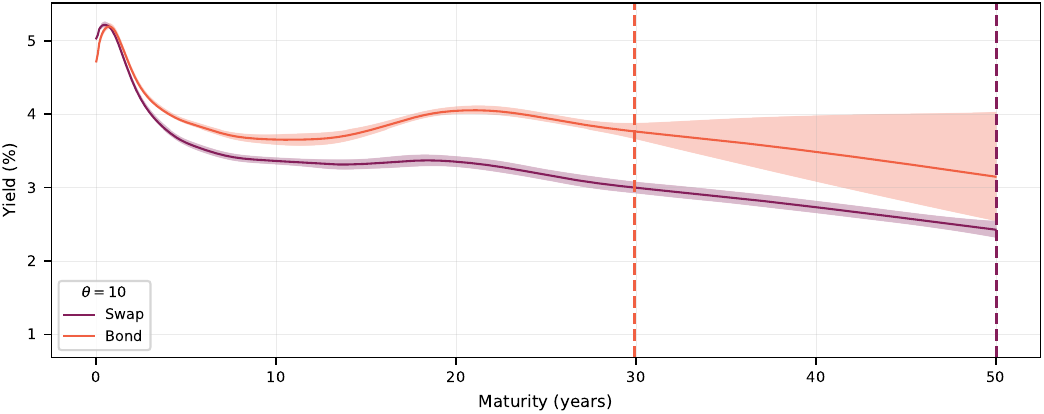}
  \end{subfigure}\hfill
  \begin{subfigure}[t]{0.49\linewidth}
    \centering
    \includegraphics[width=\linewidth]{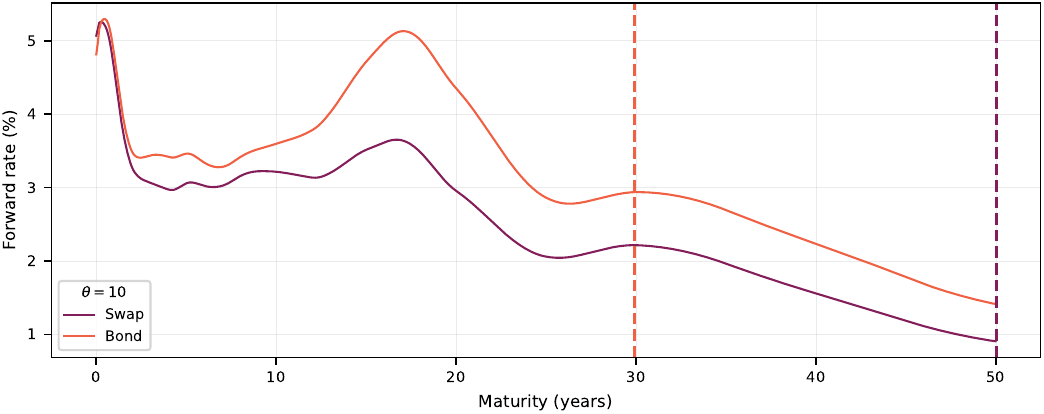}
  \end{subfigure}\vspace{0.3em}
    \begin{subfigure}[t]{0.49\linewidth}
    \centering
    \includegraphics[width=\linewidth]{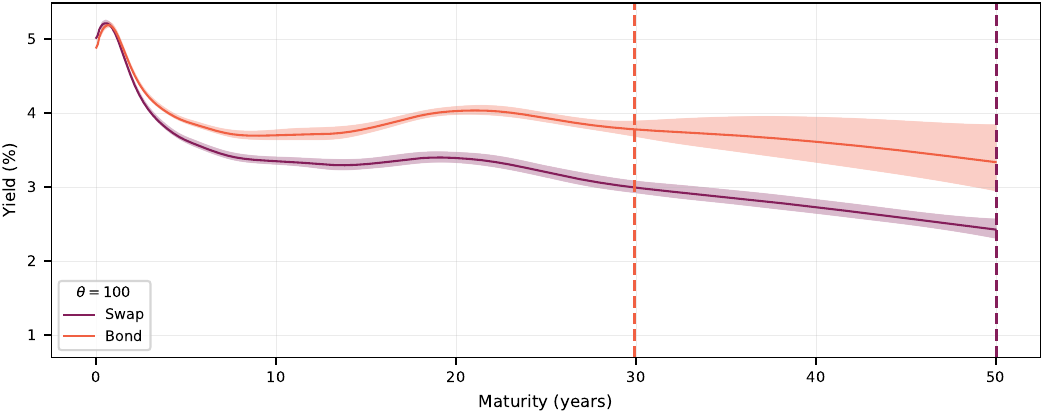}
  \end{subfigure}\hfill
  \begin{subfigure}[t]{0.49\linewidth}
    \centering
    \includegraphics[width=\linewidth]{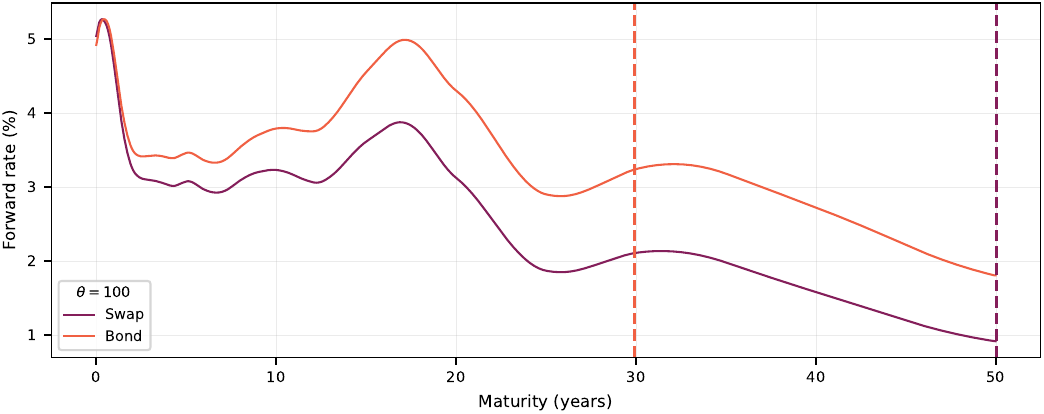}
  \end{subfigure}\vspace{0.3em}
    \begin{subfigure}[t]{0.49\linewidth}
    \centering
    \includegraphics[width=\linewidth]{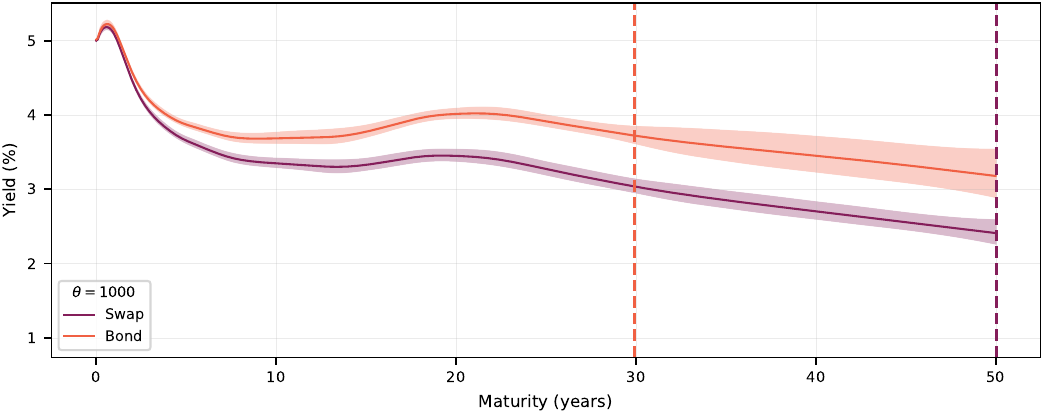}
  \end{subfigure}\hfill
  \begin{subfigure}[t]{0.49\linewidth}
    \centering
    \includegraphics[width=\linewidth]{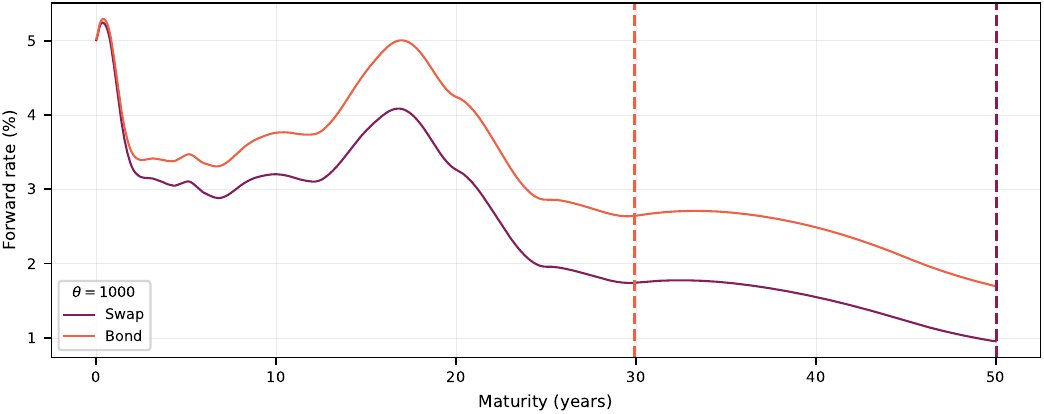}
  \end{subfigure}\vspace{0.3em}
        \bnotefig{This figure shows the resulting yield curves for various $\theta$ on the left and respective forward rate curves on the right on 2023-06-15. In all panels, the vertical dashed lines indicate the longest available data point in the respective product class. The shaded areas show the $3\sigma$ confidence bands derived from the Gaussian process view and are capped at $\pm 2\%$. All values are in \%} 
   \label{fig:example_day_2023}
\end{figure}

\end{appendix}
\end{document}